\documentclass[acmtog,nonacm]{acmart}
\setcopyright{none}
\authorsaddresses{}
\providecommand{\bm}[1]{\boldsymbol{#1}}
\usepackage{placeins}
\usepackage{multirow}
\usepackage{capt-of}
\usepackage{enumitem}
\usepackage{needspace}
\usepackage{xurl}
\title{CAST: Real-Time Motion Capture for Any Skeleton Topology}
\author{Pu Pang}
\affiliation{\institution{Xi'an Jiaotong University}\country{China}}
\affiliation{\institution{Zhongguancun Academy}\country{China}}

\author{Zhongfei Qing}
\affiliation{\institution{Tencent}\country{China}}

\author{Ling-Hao Chen}
\affiliation{\institution{Tsinghua University}\country{China}}

\author{Songnan Li}
\affiliation{\institution{Tencent}\country{China}}

\author{Hua Huang}
\authornote{Corresponding author.}
\affiliation{\institution{Beijing Normal University}\country{China}}

\author{Bo Hu}
\authornote{Project leader.}
\affiliation{\institution{Tencent}\country{China}}

\renewcommand{\shortauthors}{Pang et al.}

\begin{document}
\raggedbottom

\begin{abstract}
This paper studies the problem of recovering motion from monocular video, even when the target skeleton topology differs from the subject observed in the video.
Although previous methods have made efforts to resolve this problem, they remain vulnerable to joint-rotation errors that are coupled across the skeleton and can propagate along its kinematic hierarchy.
To mitigate these errors, we present CAST, a topology-aware motion capture framework that combines learned local rotations with geometric correction guided by predicted global bone directions.
Our observability analysis characterizes which rotational components can be constrained by outgoing bone directions. Guided by this analysis, GRACE uses reliability-gated alignment to correct composed orientations in parent-to-child order, leaving unconstrained components to the learned rotation predictor.
The motion predictor and correction module are trained jointly with motion-derived supervision, using a frozen image encoder.
Experiments under the TopoCap evaluation protocol show that the proposed framework improves motion recovery across different target topologies, including unseen structures.
Reference-based cross-skeleton evaluation on a curated Mixamo benchmark and a user preference study further assess transfer to different target structures.
A causal CAST-B variant achieves over 51 FPS on a single NVIDIA A100 GPU, with timing covering image encoding and motion prediction.
Code repository: \url{https://github.com/fankewen/CAST}.
\end{abstract}

\begin{teaserfigure}
\centering
\includegraphics[width=\linewidth]{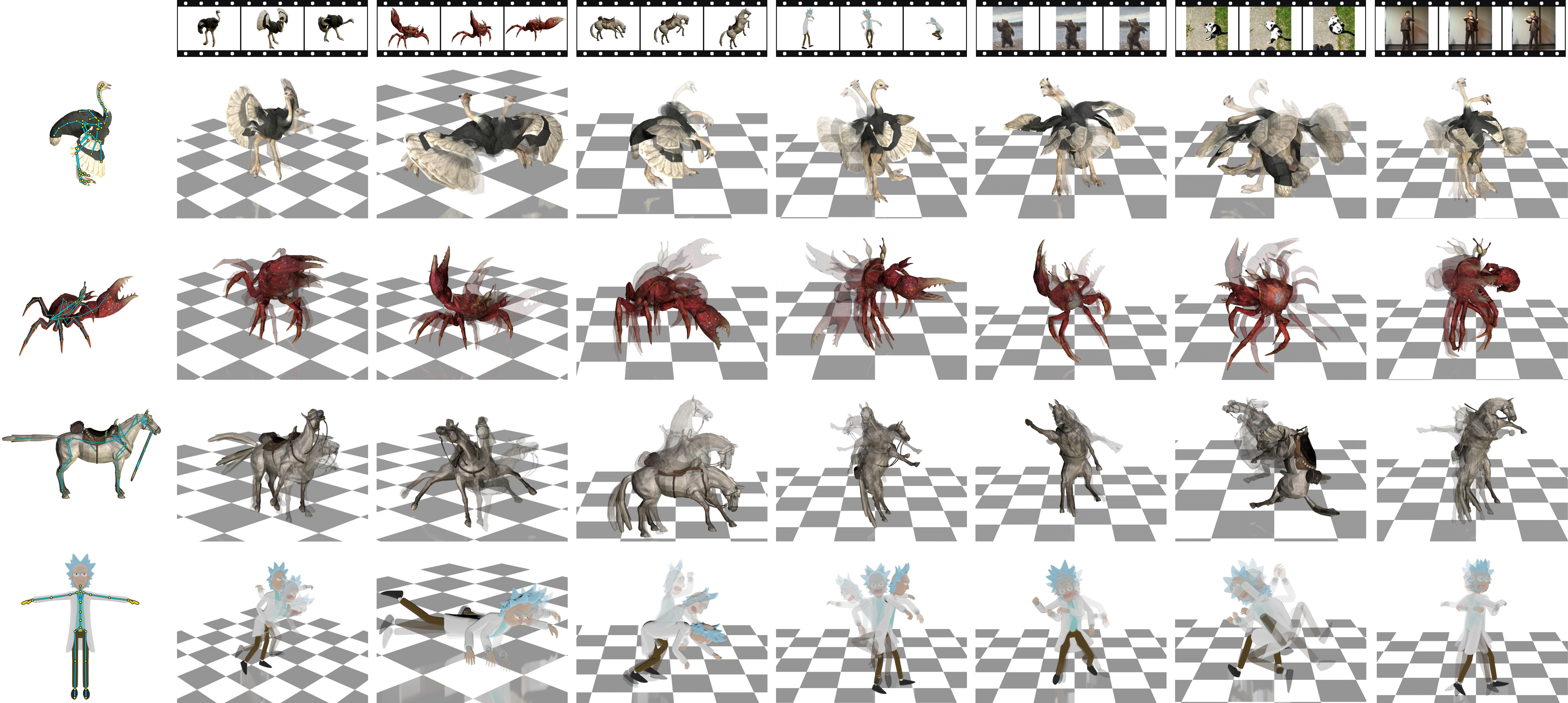}
\caption{CAST animates four target skeletons (left) from each of seven input clips (top).}
\label{fig:teaser}
\Description{Seven input video clips drive four different target skeletons using CAST.}
\end{teaserfigure}

\maketitle

\section{Introduction}
\label{sec:intro}

Motion capture animates digital characters in games, films, virtual avatars, and interactive applications~\citep{zhang2024smt,yang2025star}.
AI-generated videos broaden this task to creatures and stylized characters whose implicit articulated structures may differ from the rigs used for animation~\citep{hu2024animateanyone,men2025mimo,wang2026poseanything}.
This calls for recovering motion on a supplied target skeleton, even when its topology differs from that of the video subject.
However, most video-to-motion systems assume fixed human or animal templates, known joint correspondences, and predefined rotation conventions~\citep{motionbert2022,shin2024wham,shen2024gvhmr}, limiting their applicability to targets with different joint counts, branching structures, or rest geometries.

Video-driven 4D-mesh methods~\citep{wu2026rdmesh} recover vertex trajectories rather than rotations for an existing target rig.
Estimating source-skeleton motion and then retargeting it typically requires source--target bone correspondences or target motion examples~\citep{aberman2020skeleton,chen2025motion2motion}.
Direct target-skeleton methods avoid this explicit retargeting step. MoCapAnything~\citep{gong2025mocapanything} fits rotations to predicted joint trajectories using inverse kinematics (IK), while TopoCap~\citep{pu2026topocap} derives them from predicted global rotations. MoCapAnything V2~\citep{gong2026mocapanythingv2} replaces IK with a learned, reference-conditioned pose-to-rotation decoder.

However, arbitrary-skeleton motion capture presents two related challenges, as joint rotations are inherently ambiguous in monocular video and the target skeleton's kinematic hierarchy couples errors across joints.
Pose-based inverse kinematics fits rotations to predicted joint positions, but noisy trajectories can introduce fitting residuals, while joint positions alone do not fully constrain axial twist~\citep{gong2025mocapanything}.
Learned local-rotation prediction faces a different challenge because an error in a parent's rotation alters the global orientations of its descendants and can propagate along the kinematic chain~\citep{pavllo2018quaternet,pan2024romo}.
Predicting global orientations avoids recursive local-rotation composition during estimation, but recovering accurate parent-relative rotations still requires consistent orientation estimates across neighboring joints~\citep{pu2026topocap}.
The impact of these errors on downstream poses and end-effector positions depends on the target skeleton's chain depth, branching structure, and bone lengths.

The target skeleton determines not only error propagation but also which rotational components outgoing bone directions can constrain.
We therefore introduce CAST, a motion Capture system for Arbitrary Skeleton Topologies that combines learned local rotations with topology-dependent geometric correction.
CAST predicts initial local rotations and global bone directions with associated reliabilities; the latter provide geometric constraints without recursive local-rotation composition.
Our Gated Rotation Alignment and CorrEction (GRACE) module uses reliability-weighted alignment and a learned gate to correct composed global rotations while leaving unconstrained components to the rotation predictor.
Corrected parent orientations are then used to compose child rotations, allowing correction to intervene during hierarchical propagation.

Our contributions are as follows:
\begin{itemize}[leftmargin=*]
    \item We introduce CAST, a unified motion capture framework for diverse skeleton topologies. CAST-L reduces MPJPE by 27.6--62.5\% on matched-rig benchmarks and by 54.7\% on the cross-skeleton benchmark relative to our reimplementation of MoCapAnything V2.
    \item We establish the relationship between bone-direction constraints and rotation recoverability and use this analysis to develop GRACE. The full GRACE configuration reduces MPJPE by 8.6--15.3\% across different visual backbones and dataset splits.
    \item We extend CAST to real-time streaming motion capture and motion completion from sparse observations, with the streaming variant exceeding 51 FPS on a single A100 GPU. We also demonstrate sketch-guided in-betweening and provide a Blender add-on for real-time motion capture.
\end{itemize}

\section{Related Work}
\label{sec:related}

\paragraph{Video-based motion capture for arbitrary skeletons.}
Early video-based motion capture methods typically rely on fixed human or animal body models.
MotionBERT~\citep{motionbert2022} learns general human-motion representations, while WHAM and GVHMR~\citep{shin2024wham,shen2024gvhmr} recover world-grounded human motion with body-specific priors.
Recent work extends video-based capture to arbitrary rigged assets.
MoCapAnything~\citep{gong2025mocapanything} predicts target-joint trajectories and recovers asset-specific rotations with inverse kinematics.
MoCapAnything V2~\citep{gong2026mocapanythingv2} replaces this analytical stage with a learnable, reference-conditioned pose-to-rotation module and jointly trains the video-to-pose and pose-to-rotation paths.
These methods differ in how they represent the intermediate motion and how the target skeleton enters the recovery process, ranging from joint positions and inverse kinematics to directly learned local rotations.

\paragraph{Cross-skeleton motion transfer and topology-agnostic representations.}
Motion retargeting methods learn to transfer motion between characters with different structures using skeleton-aware operators or shared motion representations.
Neural kinematic networks and skeleton-aware networks learn motion mappings across character structures~\citep{villegas2018nkn,aberman2020skeleton}, while HuMoT uses a topology-agnostic transformer representation for character animation retargeting~\citep{mourot2023humot}.
More recent approaches use topology-conditioned generation or shared structural representations, including AnyTop~\citep{gat2025anytop}, TopoCap~\citep{pu2026topocap}, and SAMoR~\citep{zhang2026samor}.
Motion2Motion~\citep{chen2025motion2motion} transfers motion across topologies using sparse source--target correspondences and target motion examples.
Other video-guided asset animation methods, such as Puppeteer~\citep{song2025puppeteer} and R-DMesh~\citep{wu2026rdmesh}, use rigging or dynamic-mesh representations to connect visual input with animation.
These approaches differ in their input modality, use of source--target correspondence, and output representation.
CAST focuses on RGB-conditioned motion capture with an existing target skeleton and predicts local rotations for that skeleton without requiring an explicit source--target joint map.

\section{Problem Setup}
\label{sec:setup}

Given an RGB video $\mathcal I_{1:T}$ and a target skeleton $\mathcal S=\allowbreak(\mathcal V,\allowbreak\mathcal E,\allowbreak p,\allowbreak\mathbf o,\allowbreak\mathbf L^{\mathrm{rest}})$, we predict per-frame local joint rotations $\mathbf L_{t,j}\in\mathtt{SO}(3)$ that drive the target rig through the observed motion.
The skeleton is a rooted tree with joints $\mathcal V=\{0,\ldots,J-1\}$ and parent map $p(\cdot)$, where $p(0)=-1$.
Its rest pose is specified by parent-local offsets $\mathbf o_j\in\mathbb R^3$ and local joint rotations $\mathbf L^{\mathrm{rest}}_j\in\mathtt{SO}(3)$.
In this paper, we consider two instances of this problem.
For \emph{matched-rig capture}, the video depicts the supplied target rig.
For \emph{cross-skeleton animation}, the video depicts a different source rig from the supplied target skeleton, and no source--target joint correspondence is assumed.

\section{Method}
\label{sec:method}

Figure~\ref{fig:overview}(a) presents CAST, our framework for motion capture with topology-dependent rotation correction.
Its central component, GRACE, uses predicted global bone directions to correct rotation errors propagated through the skeleton hierarchy.
The geometry of a joint's outgoing bones determines the rotational constraints available for correction, while directional reliability and a learned gate control the updates.
GRACE applies these updates in parent-to-child order, incorporating each corrected orientation into the subsequent rotation composition.
To provide the initial local rotations and features for direction prediction, we adapt the prediction backbone of MoCapAnything~V2~\citep{gong2026mocapanythingv2} and train it jointly with GRACE.
The backbone architecture is detailed in Appendix~\ref{app:architecture}.

\begin{figure*}[!t]
\centering
\includegraphics[width=\linewidth]{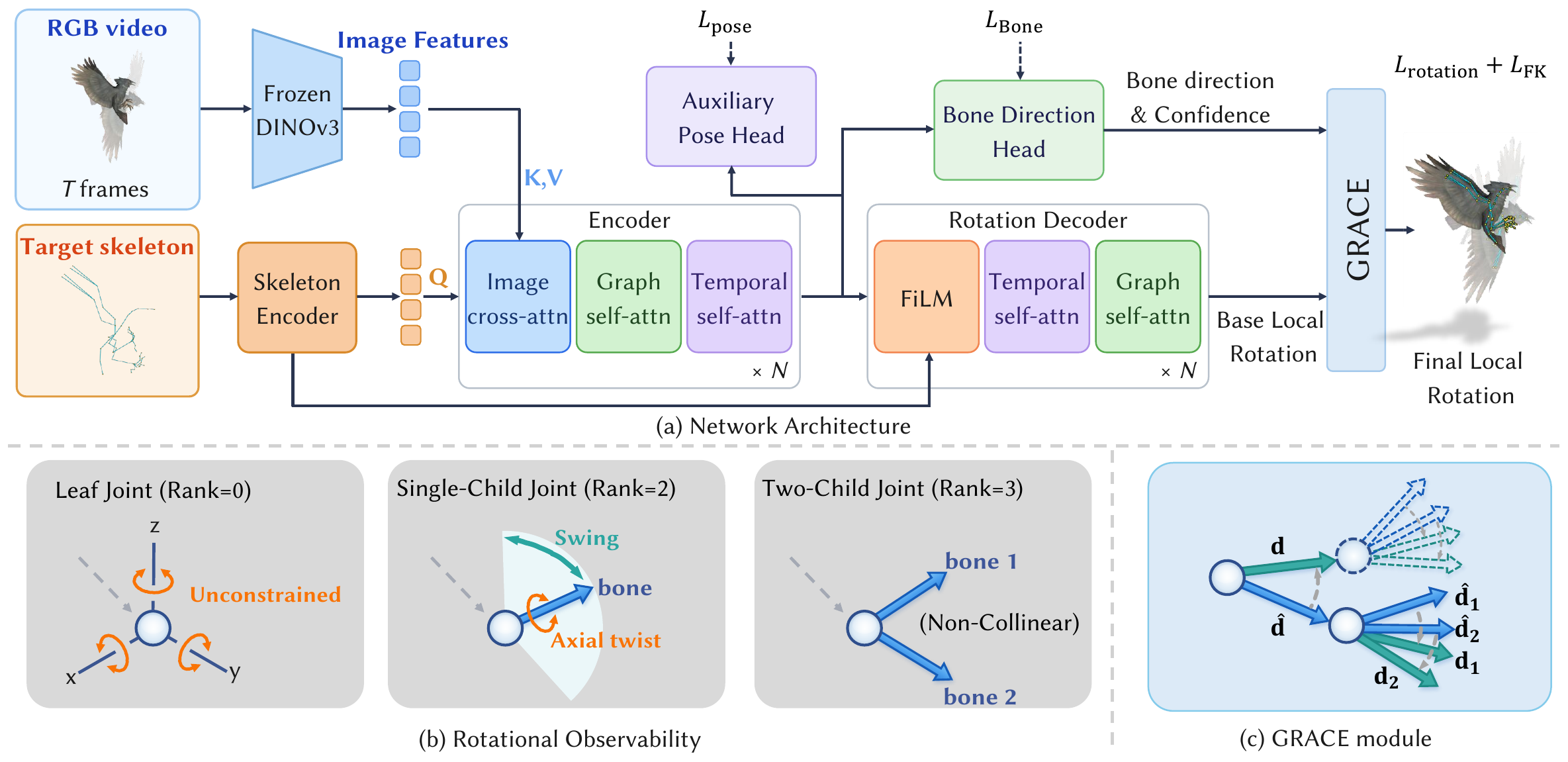}
\caption{CAST overview.
(a) CAST predicts initial local rotations and global bone directions from an RGB video and a target skeleton, then refines the rotations with GRACE.
(b) Rotational observability: a leaf provides no directional constraint, a single child constrains swing but not axial twist, and two non-collinear child bones constrain the full rotation.
(c) GRACE aligns FK directions $\widehat{\mathbf d}$ (blue) with predicted directions $\mathbf d$ (teal) through minimum-angle swing for a single direction or joint alignment of multiple directions.}
\label{fig:overview}
\end{figure*}

\subsection{Topology-dependent observability and error correction}
\label{sec:cast}

\paragraph{Rotation composition.}
As shown in Figure~\ref{fig:overview}(a), the method extends from the backbone of MoCapAnything~V2~\citep{gong2026mocapanythingv2}, predicting the initial local prediction $\widehat{\mathbf L}^{0}_{t,i}$ of joint $i$. 
For joint $i$ with parent $p(i)$, composing its local rotation with the parent's global rotation gives $\widehat{\mathbf G}_{t,i}$, which transforms its outgoing rest-pose bone directions into global directions $\widehat{\mathbf d}_{t,ij}$,
\begin{equation}
\widehat{\mathbf G}_{t,i}=
\begin{cases}
\widehat{\mathbf L}^{0}_{t,0},&i=0,\\
\widehat{\mathbf G}_{t,p(i)}\widehat{\mathbf L}^{0}_{t,i},&i>0,
\end{cases}
\qquad
\widehat{\mathbf d}_{t,ij}=\widehat{\mathbf G}_{t,i}\mathbf r_{ij}.
\label{eq:recurrence}
\end{equation}
Here $\mathbf r_{ij}=\mathbf o_j/\lVert\mathbf o_j\rVert_2$ is the unit rest-pose direction of bone $i\to j$ in joint $i$'s local axes, assuming $\lVert\mathbf o_j\rVert_2>0$ in the directional analysis, and $\widehat{\mathbf d}_{t,ij}\in\mathbb S^2$ is the composed bone direction before correcting joint $i$. The hat marks quantities before correction: $\widehat{\mathbf L}^{0}$ for the initial local rotations, and $\widehat{\mathbf G}$ for global rotations and forward-kinematic (FK) directions.

\paragraph{Ideal geometric correction.}
Let $\mathcal C(i)=\{j \mid p(j)=i\}$ denote the children of $i$.
As there is a gap between the composed direction $\widehat{\mathbf d}_{t,ij}$ and ground-truth direction ${\mathbf d}^{\mathrm{gt}}_{t,ij}$, we introduce a rotation factor $\mathbf C\in\mathtt{SO}(3)$ to correct it.
We define the ideal correction as
\begin{equation}
\mathbf C^{\mathrm{ideal}}_{t,i}=
\underset{\mathbf C\in\mathtt{SO}(3)}{\arg\min}\;
\sum_{j\in\mathcal C(i)}
\lVert\mathbf C\widehat{\mathbf d}_{t,ij}-\mathbf d^{\mathrm{gt}}_{t,ij}\rVert_2^2.
\label{eq:correction_ideal}
\end{equation}

\paragraph{Rotational observability.}
The outgoing bone geometry determines which rotational degrees of freedom can be constrained in Eq.~\ref{eq:correction_ideal}.
For an infinitesimal rotation $\delta\boldsymbol\omega\in\mathbb R^3$ applied on the left of $\widehat{\mathbf G}_{t,i}$, the first-order change in each bone direction is
\begin{equation}
\delta\widehat{\mathbf d}_{t,ij}
=\delta\boldsymbol\omega\times\widehat{\mathbf d}_{t,ij}.
\label{eq:direction_jacobian}
\end{equation}
Thus, each bone constrains the two rotation components perpendicular to its axis, represented by the projector $\mathbf I-\widehat{\mathbf d}_{t,ij}\widehat{\mathbf d}_{t,ij}^{\mathsf T}$, where $\mathbf I$ is the $3\times3$ identity matrix.

For nonempty $\mathcal C(i)$, averaging these per-bone projectors gives the normalized Gauss--Newton matrix
\begin{equation}
\begin{aligned}
\mathbf A^{\mathrm{glob}}_{t,i}
&=\frac{1}{|\mathcal C(i)|}\sum_{j\in\mathcal C(i)}
\bigl(\mathbf I-\widehat{\mathbf d}_{t,ij}\widehat{\mathbf d}_{t,ij}^{\mathsf T}\bigr)
\\
&=\widehat{\mathbf G}_{t,i}
\underbrace{\left[
\frac{1}{|\mathcal C(i)|}\sum_{j\in\mathcal C(i)}
\bigl(\mathbf I-\mathbf r_{ij}\mathbf r_{ij}^{\mathsf T}\bigr)
\right]}_{\mathbf A_i^{\mathrm{loc}}}
\widehat{\mathbf G}_{t,i}^{\mathsf T}.
\end{aligned}
\label{eq:direction_information}
\end{equation}
\begin{definition}[Rotational observability]
\label{def:rotational_observability}
The locally observable rotational subspace of joint $i$, expressed in its rest-pose local axes, is defined as the range of $\mathbf A_i^{\mathrm{loc}}$.
Its dimension, $\operatorname{rank}(\mathbf A_i^{\mathrm{loc}})$, is the number of observable rotational degrees of freedom constrained by its outgoing bone directions.
\end{definition}
Equation~\ref{eq:direction_information} rotates this representation into the current pose without changing its rank, so the number of observable rotational degrees of freedom is independent of the joint's orientation.
Figure~\ref{fig:overview}(b) illustrates three cases: for leaves, we define both matrices as zero, giving rank zero; a single direction or collinear directions give rank two, constraining swing but leaving axial twist unobservable; two or more non-collinear directions give rank three and constrain full rotation. Derivation in Appendix~\ref{app:observability_derivation}.

\begin{lemma}[Uniqueness of the ideal correction]
\label{lem:ideal_correction}
Suppose all ground-truth outgoing bone directions at joint $i$ are obtained by applying the same rotation $\mathbf G^{\mathrm{gt}}_{t,i}$ to their corresponding rest-pose directions.
Then Eq.~\ref{eq:correction_ideal} has a unique solution if and only if $\operatorname{rank}(\mathbf A_i^{\mathrm{loc}})=3$, in which case $\mathbf C^{\mathrm{ideal}}_{t,i}=\mathbf G^{\mathrm{gt}}_{t,i}\widehat{\mathbf G}_{t,i}^{\mathsf T}$ (proof in Appendix~\ref{app:error_cancellation}).
\end{lemma}
In the fully observable case, form the cross-covariance matrix and its SVD:
\begin{equation}
\mathbf H^{\mathrm{ideal}}_{t,i}
=\sum\nolimits_{j\in\mathcal C(i)}\mathbf d^{\mathrm{gt}}_{t,ij}\widehat{\mathbf d}_{t,ij}^{\mathsf T}
=\mathbf U_{t,i}\boldsymbol\Sigma_{t,i}\mathbf V_{t,i}^{\mathsf T},
\label{eq:ideal_cross_covariance}
\end{equation}
where $\mathbf U_{t,i}$ and $\mathbf V_{t,i}$ are orthogonal and $\boldsymbol\Sigma_{t,i}$ contains the singular values in non-increasing order.
The Wahba/Kabsch solution~\citep{wahba1965,kabsch1976} is
\begin{equation}
\mathbf C^{\mathrm{ideal}}_{t,i}
=\mathbf U_{t,i}\operatorname{diag}(1,1,\det(\mathbf U_{t,i}\mathbf V_{t,i}^{\mathsf T}))\mathbf V_{t,i}^{\mathsf T}.
\label{eq:ideal_closed_form}
\end{equation}
Applying this solution gives the corrected global rotation $\mathbf G_{t,i}=\mathbf C^{\mathrm{ideal}}_{t,i}\widehat{\mathbf G}_{t,i}$.
Under the conditions of Lemma~\ref{lem:ideal_correction}, $\mathbf G_{t,i}=\mathbf G^{\mathrm{gt}}_{t,i}$, so the error in the composed global rotation is removed.

With compatible directions but rank two, we select a minimum-angle swing that aligns the bone directions; axial twist remains undetermined (Appendix~\ref{app:alignment}).
For rank zero, no bone direction constrains the correction, and we set $\mathbf C^{\mathrm{ideal}}_{t,i}=\mathbf I$.

\subsection{GRACE: reliability-gated rotation correction}
\label{sec:cast_gating}

\paragraph{Confidence-weighted alignment.}
In practice, GRACE replaces the ground-truth directions $\mathbf d^{\mathrm{gt}}_{t,ij}$ assumed in Section~\ref{sec:cast} with predicted directions $\mathbf d_{t,ij}$.
The bone direction head uses a 7-layer MLP to fuse parent--child features and bone geometry, predicting global unit directions and reliabilities $c_{t,ij}\in[0,1]$.
We align the FK-derived directions $\widehat{\mathbf d}_{t,ij}=\widehat{\mathbf G}_{t,i}\mathbf r_{ij}$ with these predictions, weighting each bone by its reliability,
\begin{equation}
\mathbf C^*_{t,i}=\underset{\mathbf C\in\mathtt{SO}(3)}{\arg\min}\;
\sum\nolimits_{j\in\mathcal C(i)}c_{t,ij}
\lVert\mathbf C\widehat{\mathbf d}_{t,ij}-\mathbf d_{t,ij}\rVert_2^2.
\label{eq:correction}
\end{equation}
We compute $\mathbf C^*_{t,i}$ numerically using stabilized Wahba/Kabsch alignment~\citep{wahba1965,kabsch1976} for well-conditioned multi-direction inputs.
Single-direction inputs use a minimum-angle swing.
Figure~\ref{fig:overview}(c) illustrates the alignment, with solver details provided in Appendices~\ref{app:alignment} and~\ref{app:numerics}.

\paragraph{Gated correction.}
To prevent inaccurate direction predictions from disrupting early training, we use a learned gate initialized to favor small corrections.
Joint features, directional reliability and disagreement, and target-skeleton geometry determine its strength $g_{t,i}\in[0,1]$, which scales the alignment proposal in the Lie algebra,
\begin{equation}
\widetilde{\mathbf C}_{t,i}
=\mathtt{exp}\!\left(g_{t,i}\mathtt{log}(\mathbf C^*_{t,i})\right).
\label{eq:lie_gate}
\end{equation}
Here $\widetilde{\mathbf C}_{t,i}$ is the applied correction, and $\mathtt{log}$ and $\mathtt{exp}$ denote the rotation logarithm and matrix exponential.
Gate and logarithm details are given in Appendix~\ref{app:numerics}.

GRACE corrects rotations level by level from the root to the leaves.
The root starts from $\widehat{\mathbf G}_{t,0}=\widehat{\mathbf L}^{0}_{t,0}$.
Each non-root joint composes its initial local rotation with its parent's corrected global rotation, giving $\widehat{\mathbf G}_{t,i}=\mathbf G_{t,p(i)}\widehat{\mathbf L}^{0}_{t,i}$.
At each level, we recompute the outgoing FK directions, solve the alignment in Eq.~\ref{eq:correction}, and left-multiply each provisional rotation by its gated correction $\widetilde{\mathbf C}_{t,i}$.
The corrected global rotations then serve as the parent rotations for the next level.
Finally, we recover local rotations as $\mathbf L_{t,i}=\mathbf G_{t,p(i)}^{\mathsf T}\mathbf G_{t,i}$ for $i>0$, with $\mathbf L_{t,0}=\mathbf G_{t,0}$.

\subsection{Training objective}
\label{sec:training}

We directly supervise the bone direction head with a direction loss $\mathcal L_{\mathrm{dir}}$ and a reliability loss $\mathcal L_{\mathrm{rel}}$:
\begingroup
\allowdisplaybreaks[4]
\begin{align}
\mathcal L_{\mathrm{dir}}&=\mathbb E\!\left[1-\mathbf d_{t,ij}^{\mathsf T}\mathbf d^{\mathrm{gt}}_{t,ij}\right],
\label{eq:direction_loss}\\
\mathcal L_{\mathrm{rel}}&=\mathbb E\!\left[\rho\!\left(c_{t,ij}-\operatorname{sg}\!\left[\exp\!\left(-\alpha_{t,ij}/\tau\right)\right]\right)\right],
\label{eq:reliability_target}\\
\alpha_{t,ij}&=\arccos(\mathbf d_{t,ij}^{\mathsf T}\mathbf d^{\mathrm{gt}}_{t,ij}),
\label{eq:direction_angular_error}
\end{align}
\endgroup
where $\alpha_{t,ij}$ is the angular error in radians, $\rho$ is the Smooth L1 penalty, $\tau>0$ controls the target's decay, and $\operatorname{sg}(\cdot)$ denotes the stop-gradient operator.
We apply $\mathcal L_{\mathrm{rotation}}$ and $\mathcal L_{\mathrm{FK}}$ to the corrected local rotations $\mathbf L_{t,i}$ and their FK positions, respectively, with no separate loss on the initial predictions.
The auxiliary pose loss $\mathcal L_{\mathrm{pose}}$ supervises the pose head's joint positions and adjacent-frame displacements.
Combining these terms with $\mathcal L_{\mathrm{bone}}=\mathcal L_{\mathrm{dir}}+0.1\mathcal L_{\mathrm{rel}}$ gives
\begin{equation}
\mathcal L=\mathcal L_{\mathrm{pose}}
+\mathcal L_{\mathrm{rotation}}
+\lambda_{\mathrm{FK}}\mathcal L_{\mathrm{FK}}
+\lambda_{\mathrm{bone}}\mathcal L_{\mathrm{bone}}.
\label{eq:objective}
\end{equation}
The component losses, weights, and reductions are detailed in Appendix~\ref{app:training}.

\section{Experiments}
\label{sec:exp}

\begin{table*}[!t]
\caption{Combined matched-rig and cross-skeleton results.
Lower is better for all metrics, and a dash denotes an unreported result.
MPJPE and PA-MPJPE are reported in millimeters, MPJVE in millimeters per frame, and CD in normalized coordinates.
}
\label{tab:combined}
\centering
\setlength{\tabcolsep}{6pt}
\begin{tabular*}{\linewidth}{@{\extracolsep{\fill}}lllrrrr@{}}
\toprule
Task&Dataset&Model&MPJPE&PA-MPJPE&MPJVE&CD\\\midrule
\multirow{15}{*}{Matched-rig}&\multirow{5}{*}{Zoo}&Puppeteer&454.70&236.38&19.74&0.245\\
&&TopoCap&--&--&--&--\\
&&MCA V2$^{\ast}$&105.52&67.04&10.35&0.077\\
&&CAST-B&42.68&33.23&5.48&0.036\\
&&CAST-L&\textbf{39.59}&\textbf{31.05}&\textbf{5.14}&\textbf{0.034}\\\cmidrule(lr){2-7}
&\multirow{5}{*}{MObj Seen}&Puppeteer&205.85&103.89&22.09&0.136\\
&&TopoCap&157.51&74.04&34.86&0.096\\
&&MCA V2$^{\ast}$&94.38&58.62&16.71&0.071\\
&&CAST-B&55.39&32.92&9.73&0.043\\
&&CAST-L&\textbf{52.19}&\textbf{30.91}&\textbf{9.10}&\textbf{0.041}\\\cmidrule(lr){2-7}
&\multirow{5}{*}{MObj Unseen}&Puppeteer&181.63&93.23&21.10&0.119\\
&&TopoCap&191.77&82.11&38.93&0.109\\
&&MCA V2$^{\ast}$&103.04&61.98&16.44&0.074\\
&&CAST-B&79.21&46.52&10.73&0.057\\
&&CAST-L&\textbf{74.63}&\textbf{44.21}&\textbf{10.19}&\textbf{0.054}\\\midrule
\multirow{4}{*}{Cross-skeleton}&\multirow{4}{*}{Mixamo}&TopoCap&363.59&154.72&57.20&0.174\\
&&MCA V2$^{\ast}$&277.35&127.38&33.64&0.135\\
&&CAST-B&149.50&85.79&19.60&0.087\\
&&CAST-L&\textbf{125.62}&\textbf{71.16}&\textbf{15.74}&\textbf{0.074}\\\bottomrule
\end{tabular*}
\end{table*}

\begin{figure*}[!t]
\centering
\includegraphics[width=1.0\linewidth]{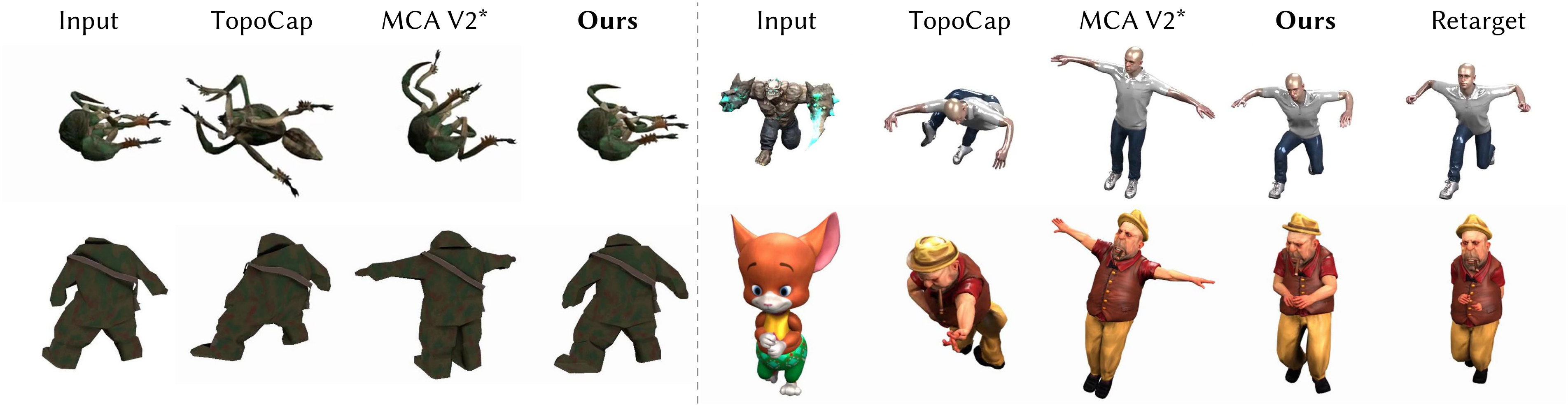}
\caption{Qualitative comparison with baselines.
\textbf{Left:} matched-rig motion capture.
\textbf{Right:} cross-skeleton animation, with retargeted target-character references in the final column (Retarget).
Each example shows an input frame and the three methods' outputs on the same target character.}
\label{fig:qualitative_results}
\end{figure*}

\subsection{Experimental setup}

\paragraph{Datasets.}
We train CAST jointly on Truebones Zoo (Zoo)~\citep{truebones2022} and MObjaverse (MObj), using a 95\%/5\% train/test split for Zoo and TopoCap's Train/Seen/Unseen splits for MObj~\citep{pu2026topocap,deitke2023objaversexl}.
Cross-skeleton evaluation uses 60 held-out Mixamo animations~\citep{adobe2026mixamo}, ten per character skeleton across six skeletons.
Pairing each source sequence with the other five skeletons yields 300 ordered source--target pairs.
Mixamo-retargeted animations on the target characters serve as references used only for evaluation.
Data and evaluation details are given in Appendices~\ref{app:data} and~\ref{app:evaluation_protocol}.

\paragraph{Implementation details.}
We evaluate two model sizes. CAST-B uses 8 blocks of width 256 in both the encoder and rotation decoder, while CAST-L uses 16 blocks of width 384.
Both use frozen DINOv3~\citep{simoni2025dinov3} image features by default. 
Ablation and task-specific experiments use CAST-B unless noted otherwise.
We train for 6{,}000 steps with AdamW on 4 NVIDIA-A100 GPUs, using a batch size of 256 and 81-frame clips. 
The learning-rate schedule and loss weights are given in Appendix~\ref{app:training}.
For baselines, we use the TopoCap~\citep{pu2026topocap} official checkpoint and reimplement MoCapAnything V2~\citep{gong2026mocapanythingv2}, denoted by MCA V2$^{\ast}$, training it on the same Zoo and MObj data as CAST.
We also evaluate Puppeteer~\citep{song2025puppeteer} for matched-rig motion capture using its released video-guided optimizer and pretrained tracking, depth, and optical-flow models on the supplied rigged assets.

\paragraph{Metrics.}
For both tasks, we report MPJPE, PA-MPJPE, MPJVE, and Chamfer distance (CD) on root-centered predictions and references over complete animations.
Following TopoCap's unit-bounding-box normalization and reporting convention~\citep{pu2026topocap}, MPJPE and PA-MPJPE are reported in millimeters, MPJVE in millimeters per frame, and CD in normalized coordinates.
Metric definitions and temporal-window handling are detailed in Appendices~\ref{app:metrics} and~\ref{app:evaluation_protocol}.

\subsection{Motion capture and cross-skeleton animation}
\label{sec:main_results}

\paragraph{Matched-rig motion capture.}
CAST recovers the character's articulation more faithfully across both animal and humanoid rigs.
In Figure~\ref{fig:qualitative_results} (left), CAST-L preserves the folded leg configuration of the multi-limbed character, whereas the baselines spread or raise limbs that are folded in the input.
For the humanoid, it also avoids the outstretched-arm pose produced by MCA V2$^{\ast}$.
These differences concern the relative arrangement of body parts, which an overall orientation adjustment alone cannot resolve.
The quantitative results support this observation: CAST-L gives the lowest reported errors on all four metrics across Zoo, MObj Seen, and MObj Unseen (Table~\ref{tab:combined}), with MPJPE reductions of 27.6--62.5\% relative to MCA V2$^{\ast}$.
The improvements persist after Procrustes alignment, indicating better articulated pose recovery beyond global alignment, while lower MPJVE shows that joint displacements also follow the reference motion more closely.

This behavior is consistent with CAST's use of target-skeleton queries to organize visual evidence and global bone directions to constrain the orientations produced by local-rotation composition.
The contribution of the complete GRACE configuration is examined separately in Section~\ref{sec:ablation}.
Generalization to unseen structures remains more difficult: CAST-L's MPJPE rises from 52.19 on MObj Seen to 74.63 on MObj Unseen, even though it retains an advantage over every evaluated baseline.
Thus, accommodating a variable rig does not remove the difficulty of interpreting unfamiliar articulations from images.
Additional matched-rig examples are provided in Appendix~\ref{app:qualitative_supplement}.

\paragraph{Cross-skeleton animation.}
Cross-skeleton animation requires preserving the observed action while expressing it through a different body structure.
Figure~\ref{fig:qualitative_results} (right) shows this distinction: CAST-L retains the bent arms and forward-leaning torso of the running example, and the hand-together gesture of the compact source character.
In the latter example, the retargeted reference places the arms inside the bulkier target's torso, whereas CAST-L keeps the hands in front of the body.
The action is recognizable despite a change in arm configuration, illustrating the value of conditioning motion on the target rig.
This also explains why distance to a retargeted reference captures only part of transfer quality: retargeting across different body proportions can introduce interpenetration~\citep{villegas2021contact}.
Qualitative examples of the four transfer directions are in Appendix~\ref{app:qualitative_cross_skeleton}.

Both CAST variants nevertheless improve all four reference-based metrics over the evaluated baselines (Table~\ref{tab:combined}).
CAST-L reduces MPJPE by 54.7\% relative to MCA V2$^{\ast}$, and the accompanying reductions in PA-MPJPE and MPJVE indicate gains in articulation and frame-to-frame motion as well as overall pose agreement.
Increasing capacity from CAST-B to CAST-L brings a larger MPJPE reduction here (16.0\%) than in the matched-rig evaluations (5.8--7.2\%).
This suggests that mapping visual action cues to a different target structure benefits particularly from additional model capacity in the evaluated setting.
The qualitative body adaptation and quantitative reference agreement provide complementary evidence, although neither establishes collision-free motion in general.

\begin{figure}[!htbp]
\centering
\includegraphics[width=\linewidth]{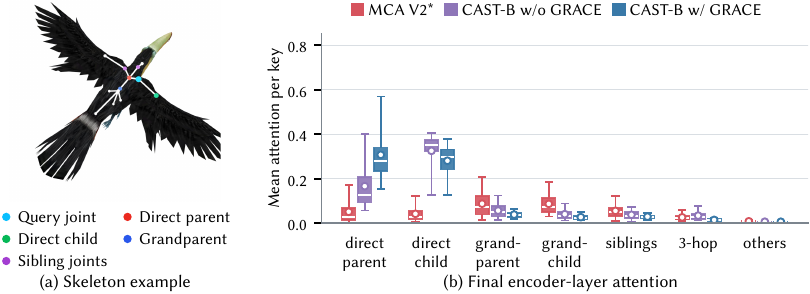}
\caption{Graph attention by kinematic relation on Zoo.
(a) An illustrative query joint and its topological relations.
(b) Final encoder-layer attention distributions pooled over valid frames and query joints across the evaluated Zoo sequences.}
\label{fig:graph_attention}
\end{figure}

\paragraph{Local and global modeling in CAST.}
For local-rotation predictors such as MCA V2$^{\ast}$, nonlocal joint context can help maintain global consistency as ancestor errors propagate through rotation composition.
GRACE reduces this burden by correcting composed orientations with predicted global bone directions, allowing the predictor to emphasize parent-relative articulation.
On Zoo, attention in the final encoder layer concentrates on direct parents and children.
Adding GRACE increases parent attention while reducing attention to distant relations (Figure~\ref{fig:graph_attention}), consistent with partial decoupling of local articulation estimation from global orientation correction.

\subsection{User study}
\begingroup
\setlength{\intextsep}{0pt}
\setlength{\columnsep}{12pt}
\begin{table}[!htbp]
\centering
\setlength{\tabcolsep}{4pt}
\setlength{\abovecaptionskip}{0pt}
\setlength{\belowcaptionskip}{4pt}
\caption{Overall user-study preferences.}
\label{tab:user_study}
\begin{tabular*}{\linewidth}{@{\extracolsep{\fill}}lrr@{}}\toprule
Method & First choice (\%) & Mean rank\\\midrule
TopoCap & 19.7 & 2.261\\
MCA V2$^{\ast}$ & 16.9 & 2.208\\
CAST-L & \textbf{63.3} & \textbf{1.531}\\
\bottomrule
\end{tabular*}
\end{table}

18 participants separately ranked 20 anonymized result triplets across both tasks by overall animation quality (protocol in Appendix~\ref{app:user_study}).
CAST-L obtains 63.3\% of first-choice votes across 360 rankings and achieves the best mean rank of 1.531 (Table~\ref{tab:user_study}), supporting the perceptual quality of its predictions.
\par
\endgroup

\subsection{Ablation study}
\label{sec:ablation}
\begin{table*}[!t]
\caption{GRACE and visual-encoder ablations on CAST-B.
Bold indicates the best result.}
\label{tab:matched_ablation}
\centering
\setlength{\tabcolsep}{6pt}
\begin{tabular*}{\linewidth}{@{\extracolsep{\fill}}lrrrrrr@{}}
\toprule
& \multicolumn{2}{c}{Zoo}
& \multicolumn{2}{c}{MObj Seen}
& \multicolumn{2}{c}{MObj Unseen} \\
Model & MPJPE & PA-MPJPE & MPJPE & PA-MPJPE & MPJPE & PA-MPJPE \\
\midrule
CAST-B w/o GRACE (DINOv2)
& \textbf{49.58} & \textbf{38.70}
& 64.56 & 37.13 & 93.32 & 52.12 \\
CAST-B w/o GRACE (DINOv3)
& 50.41 & 38.78
& \textbf{63.71} & \textbf{36.87}
& \textbf{86.64} & \textbf{50.71} \\
\midrule
CAST-B (DINOv2)
& \textbf{42.41} & \textbf{33.03}
& 57.34 & 33.27 & 83.15 & 48.19 \\
CAST-B (DINOv3)
& 42.68 & 33.23
& \textbf{55.39} & \textbf{32.92}
& \textbf{79.21} & \textbf{46.52} \\
\bottomrule
\end{tabular*}
\end{table*}

\begin{table*}[!t]
\caption{Causal inference and motion completion with CAST-B.
The Full sequence setting denotes the bidirectional model with dense input, and $k$ denotes the observation interval for key-frame completion.
Each Real-time model is trained separately for different chunk sizes.
Errors are reported in millimeters, with lower values indicating better performance.
Bold and underline mark the best and second-best results.}
\label{tab:causal}
\label{tab:completion}
\centering
\setlength{\tabcolsep}{4pt}
\begin{tabular*}{\linewidth}{@{\extracolsep{\fill}}lrrrrrr@{}}\toprule
& \multicolumn{2}{c}{Zoo} & \multicolumn{2}{c}{MObj Seen} & \multicolumn{2}{c}{MObj Unseen}\\
\cmidrule(lr){2-3}\cmidrule(lr){4-5}\cmidrule(lr){6-7}
Task setting & MPJPE & PA-MPJPE & MPJPE & PA-MPJPE & MPJPE & PA-MPJPE\\\midrule
Full sequence & \textbf{42.68} & \textbf{33.23} & 55.39 & \underline{32.92} & \textbf{79.21} & \textbf{46.52}\\
Key frame ($k{=}4$) & 45.61 & 35.15 & 58.52 & 35.84 & 85.93 & 49.65\\
Key frame ($k{=}8$) & 48.79 & 37.35 & 67.98 & 41.37 & 92.01 & 53.90\\
Key frame ($k{=}16$) & 58.53 & 43.68 & 82.73 & 50.15 & 104.45 & 61.69\\
Key frame ($k{=}32$) & 75.96 & 52.77 & 99.39 & 59.09 & 118.54 & 68.47\\
Real-time (chunk 1) & 44.09 & 34.00 & \textbf{52.46} & \textbf{32.32} & 80.79 & 47.23\\
Real-time (chunk 2) & 43.54 & 33.67 & 54.36 & 33.32 & 80.34 & 47.11\\
Real-time (chunk 4) & \underline{42.88} & \underline{33.25} & \underline{53.98} & 32.97 & \underline{79.62} & \underline{46.99}\\\bottomrule
\end{tabular*}

\end{table*}

\begin{table*}[!t]
\centering
\caption{Inference efficiency on a single NVIDIA A100.
Params (millions) and Time (ms per forward pass) include the frozen image backbone.
Frames denotes the input length per pass, and FPS measures streaming inference throughput.}
\label{tab:efficiency}
\setlength{\tabcolsep}{8pt}
\begin{tabular}{lrrrrrrrr}
\toprule
& & & \multicolumn{2}{c}{30 joints} & \multicolumn{2}{c}{52 joints} & \multicolumn{2}{c}{89 joints}\\
\cmidrule(lr){4-5}\cmidrule(lr){6-7}\cmidrule(lr){8-9}
Model & Params & Frames & Time & FPS & Time & FPS & Time & FPS\\
\midrule
TopoCap (10 steps) & 2913.8 & 8 & 321 & -- & 342 & -- & 321 & -- \\
TopoCap (300 steps) & 2913.8 & 8 & 8550 & -- & 9080 & -- & 8510 & -- \\
MCA V2 & 452.6 & 32 & 192 & -- & 191 & -- & 192 & -- \\
CAST-B & 352.6 & 81 & 142 & -- & 145 & -- & 154 & -- \\
CAST-L & 488.2 & 81 & 186 & -- & 189 & -- & 201 & -- \\
CAST-B (causal) & 352.6 & 1 & 19.5 & 51.3 & 18.5 & 53.9 & 18.9 & 52.9 \\
\bottomrule
\end{tabular}
\end{table*}

Table~\ref{tab:matched_ablation} evaluates GRACE and the choice of frozen visual encoder in CAST-B under the same training schedule.
With either DINOv2~\citep{oquab2023dinov2} or DINOv3, GRACE reduces MPJPE by 8.6--15.3\% across all splits.
The accompanying PA-MPJPE gains indicate improved articulation beyond global alignment.
DINOv3 provides its largest improvement over DINOv2 on MObj Unseen, lowering MPJPE by 4.7\% with GRACE, while performance on Zoo remains similar.
These results support DINOv3 for transfer to unfamiliar structures and show that the complete GRACE configuration provides consistent gains under both visual encoders.

\section{Application}
\label{sec:applications}

\subsection{Motion completion and causal inference}
We adapt CAST-B to key-frame motion completion and causal inference by independently fine-tuning each variant for 3{,}000 steps.
The completion model uses randomly masked visual inputs, while the three causal models use chunk sizes of 1, 2, and 4, respectively.

\paragraph{Key-frame motion completion.}
The completion variant enables sketch-guided animation inbetweening on a target skeleton (examples in Appendix~\ref{app:completion_examples}).
Using only one quarter of the input frames ($k{=}4$), the model's MPJPE is 5.7--8.5\% higher than with dense input (Table~\ref{tab:causal}).
Errors increase with longer gaps between observations.

\paragraph{Causal inference.}
Causal temporal attention allows motion capture from incoming frames.
The model trained with single-frame chunks incurs a modest accuracy loss relative to bidirectional CAST-B on Zoo and MObj Unseen (Table~\ref{tab:causal}).
The model trained with 4-frame chunks achieves comparable MPJPE and PA-MPJPE to bidirectional CAST-B across all three splits, at the cost of waiting for a complete chunk before prediction.

For live capture, we deploy the causal variant through a Blender add-on that applies the received skeletons to a rigged character as they arrive, allowing a performance to be inspected during recording.
Figure~\ref{fig:blender_addon} and Appendix~\ref{app:blender_addon} show the add-on interface and a live capture session.

\subsection{Inference efficiency}
Table~\ref{tab:efficiency} compares model size and inference efficiency on a single NVIDIA A100.
CAST-B uses fewer parameters than the evaluated baselines and processes longer temporal windows in less time, providing higher offline throughput under the reported configurations.
Scaling to CAST-L further improves motion accuracy with a moderate increase in inference cost, while retaining a throughput advantage over the baselines.
The causal variant further enables real-time online prediction across all tested skeletons without requiring future observations.
These results support both efficient offline motion capture and online use across different target skeletons.
Timings include the frozen image backbone and motion prediction, excluding video acquisition and mesh rendering.
CAST uses bf16 precision and \texttt{torch.compile}, with implementation optimizations detailed in Appendix~\ref{app:inference_optimization}.

\FloatBarrier
\section{Conclusion}
\label{sec:conclusion}

We presented CAST, a framework for recovering motion from monocular video on a supplied target skeleton, including when its topology differs from the subject in the video.
Our observability analysis characterizes which rotational components outgoing bone directions can constrain and derives closed-form corrections under exact, compatible directional evidence.
Building on this analysis, GRACE aligns composed rotations with predicted global bone directions through reliability-gated updates, propagating corrected orientations from parents to children.
CAST improves matched-rig capture on Zoo and MObj and cross-skeleton animation on the curated Mixamo benchmark over the evaluated baselines.
Ablations support the contribution of GRACE under both visual encoders, and the user study provides complementary evidence of animation quality.
Fine-tuned variants support key-frame motion completion and causal inference, with single-frame CAST-B exceeding 51 FPS in model inference on an NVIDIA A100.
Remaining limitations include rotational ambiguities unconstrained by bone directions and the inability to represent time-varying parent-local bone offsets (Appendix~\ref{sec:limitations}).
Together, these results demonstrate the value of combining learned local rotations with topology-dependent geometric correction for motion capture across skeletons.

\bibliographystyle{ACM-Reference-Format}
\bibliography{references}

\clearpage
\appendix

\section{Theoretical Analysis of GRACE}

\label{app:theory}

\subsection{Notation and forward kinematics}
\label{app:notation}
The superscript $0$ identifies the base local prediction $\widehat{\mathbf L}^0$.
The global orientation is denoted by $\widehat{\mathbf G}$ before correcting the current joint and by $\mathbf G$ after correction.
Following Section~\ref{sec:cast_gating}, GRACE uses $\widehat{\mathbf G}_{t,i}=\mathbf G_{t,p(i)}\widehat{\mathbf L}^{0}_{t,i}$ for non-root joints and $\widehat{\mathbf G}_{t,0}=\widehat{\mathbf L}^{0}_{t,0}$ for the root.
The fixed vector $\mathbf r_{ij}$ is the normalized offset of child $j$ in parent $i$'s local axes.
It has no frame index and is not the result of FK.
As in the main text, $\mathcal C(i)=\{j\mid p(j)=i\}$ denotes the children of joint $i$, and the directional analysis assumes $\lVert\mathbf o_j\rVert_2>0$.
For any consistent set of joint rotations and their composed global orientations $\mathbf G_{t,j}$, root-centered FK positions obey
\begin{equation}
\mathbf p^{\mathrm{FK}}_{t,0}=\mathbf 0,\qquad
\mathbf p^{\mathrm{FK}}_{t,j}=\mathbf p^{\mathrm{FK}}_{t,p(j)}+\mathbf G_{t,p(j)}\mathbf o_j,\quad j>0.
\label{eq:fk_positions}
\end{equation}
Thus the unit direction from $i$ to $j$ is $\mathbf G_{t,i}\mathbf r_{ij}$.
Updating joint $i$ changes its outgoing direction from $\widehat{\mathbf d}_{t,ij}=\widehat{\mathbf G}_{t,i}\mathbf r_{ij}$ to $\mathbf G_{t,i}\mathbf r_{ij}$.
Even in the rest pose, the direction obtained by FK is $\mathbf G_i^{\mathrm{rest}}\mathbf r_{ij}$, where $\mathbf G_i^{\mathrm{rest}}$ is obtained by composing the rest rotations along the tree.
The alignment target is the exact ground-truth direction $\mathbf d^{\mathrm{gt}}_{t,ij}$ in the ideal analysis of Section~\ref{sec:cast}, and the Bone Direction Head prediction $\mathbf d_{t,ij}$ in the learned correction of Section~\ref{sec:cast_gating}.
For supervision, annotated positions give $\mathbf d^{\mathrm{gt}}_{t,ij}=(\mathbf p^{\mathrm{gt}}_{t,j}-\mathbf p^{\mathrm{gt}}_{t,i})/\lVert\mathbf p^{\mathrm{gt}}_{t,j}-\mathbf p^{\mathrm{gt}}_{t,i}\rVert_2$.

\subsection{Rotational observability derivation}
\label{app:observability_derivation}
We derive the rotational constraints illustrated for the three joint types in Figure~\ref{fig:overview}(b).
Let $\widehat{\mathbf G}_{t,i}$ be a provisional global rotation and $\widehat{\mathbf d}_{t,ij}=\widehat{\mathbf G}_{t,i}\mathbf r_{ij}$ its outgoing direction.
A left multiplicative perturbation is
\begin{equation}
\mathbf G_{t,i}(\delta\boldsymbol\omega)=
\mathtt{exp}([\delta\boldsymbol\omega]_{\times})\widehat{\mathbf G}_{t,i}.
\end{equation}
Its first-order effect, Jacobian, and Jacobian Gram matrix are
\begin{equation}
\begin{aligned}
\delta\widehat{\mathbf d}_{t,ij}&=\delta\boldsymbol\omega\times\widehat{\mathbf d}_{t,ij},\\
\mathbf J_{t,ij}&=-[\widehat{\mathbf d}_{t,ij}]_\times,\\
\mathbf J_{t,ij}^{\mathsf T}\mathbf J_{t,ij}&=\mathbf I-\widehat{\mathbf d}_{t,ij}\widehat{\mathbf d}_{t,ij}^{\mathsf T}.
\end{aligned}
\label{eq:global_gram}
\end{equation}
Because $\widehat{\mathbf d}_{t,ij}=\widehat{\mathbf G}_{t,i}\mathbf r_{ij}$,
\begin{equation}
\mathbf I-\widehat{\mathbf d}_{t,ij}\widehat{\mathbf d}_{t,ij}^{\mathsf T}
=\widehat{\mathbf G}_{t,i}(\mathbf I-\mathbf r_{ij}\mathbf r_{ij}^{\mathsf T})\widehat{\mathbf G}_{t,i}^{\mathsf T}.
\label{eq:similarity}
\end{equation}
The operators formed from the composed directions $\widehat{\mathbf d}_{t,ij}$ and rest directions $\mathbf r_{ij}$ are orthogonally similar and have the same spectrum.

Using the child set $\mathcal C(i)$ defined in Appendix~\ref{app:notation} and $Z_i=|\mathcal C(i)|$, the normalized Gauss--Newton matrix is
\begin{equation}
\begin{aligned}
\mathbf A_i^{\mathrm{loc}}&=
\begin{cases}
\displaystyle \frac{1}{Z_i}\sum_{j\in\mathcal C(i)}
(\mathbf I-\mathbf r_{ij}\mathbf r_{ij}^{\mathsf T}), & Z_i>0,\\[6pt]
\mathbf 0_{3\times3}, & Z_i=0,
\end{cases}\\
\mathbf A^{\mathrm{glob}}_{t,i}&=\widehat{\mathbf G}_{t,i}\mathbf A_i^{\mathrm{loc}}\widehat{\mathbf G}_{t,i}^{\mathsf T}.
\end{aligned}
\label{eq:observability}
\end{equation}
For leaves, both matrices are zero.
For $Z_i>0$, $\mathtt{tr}(\mathbf A_i^{\mathrm{loc}})=2$, so the trace cannot distinguish different directional configurations.
All three eigenvalues, in ascending order, form the geometric descriptor $\boldsymbol\lambda_i$.

For $Z_i>0$ and any $\mathbf v\in\mathbb R^3$,
\begin{equation}
\mathbf v^{\mathsf T}\mathbf A_i^{\mathrm{loc}}\mathbf v
=\frac{1}{Z_i}\sum_{j\in\mathcal C(i)}
\lVert\mathbf r_{ij}\times\mathbf v\rVert_2^2.
\label{eq:observability_quadratic}
\end{equation}
Thus the null space consists precisely of vectors parallel to every rest direction.
A single direction gives eigenvalues $(0,1,1)$ and leaves axial twist unconstrained.
Collinear directions share that null space.
Two non-collinear directions have no common nonzero null vector, giving rank three.
A leaf has rank zero.
These statements characterize target geometry independently of the learned directions' accuracy.

\paragraph{Rank and numerical conditioning.}
For two equally weighted unit directions with angle $\theta$, the smallest eigenvalue of $\mathbf A_i^{\mathrm{loc}}$ is $(1-|\mathtt{cos}(\theta)|)/2$.
It is positive for non-collinear directions but tends to zero as they become parallel or antiparallel.
Near-collinear directions therefore remain in the rank-three case of Figure~\ref{fig:overview}(b), but can be sensitive to small directional errors.
They do not define a separate rank class or network branch.
The actual alignment fallback tests the spectrum of the confidence-weighted $\mathbf H$, not this eigenvalue directly (Appendix~\ref{app:numerics}).

\subsection{Accumulated rotation errors and ideal correction}
\label{app:error_cancellation}
We analyze the full ideal update in Section~\ref{sec:cast}, separating error propagation from the geometric conditions for its cancellation.
Throughout this subsection, the frame index $t$ is omitted, and exact directions are assumed to be compatible with the target rest geometry: $\mathbf d^{\mathrm{gt}}_{ij}=\mathbf G_i^{\mathrm{gt}}\mathbf r_{ij}$ for every edge, where the superscript $\mathrm{gt}$ denotes ground truth.

\paragraph{How local errors enter the global orientation.}
For a non-root joint, the true local rotation is $\mathbf L_i^{\mathrm{gt}}=(\mathbf G_{p(i)}^{\mathrm{gt}})^{\mathsf T}\mathbf G_i^{\mathrm{gt}}$.
Define the local prediction error and the corrected parent's global error by $\boldsymbol\Delta_i=\widehat{\mathbf L}_i^0(\mathbf L_i^{\mathrm{gt}})^{\mathsf T}$ and $\mathbf E_{p(i)}=\mathbf G_{p(i)}(\mathbf G_{p(i)}^{\mathrm{gt}})^{\mathsf T}$, respectively.
The provisional global error then obeys the exact identity
\begin{equation}
\begin{aligned}
\widehat{\mathbf E}_i
&=\widehat{\mathbf G}_i(\mathbf G_i^{\mathrm{gt}})^{\mathsf T}\\
&=\mathbf E_{p(i)}\mathbf G_{p(i)}^{\mathrm{gt}}
\boldsymbol\Delta_i(\mathbf G_{p(i)}^{\mathrm{gt}})^{\mathsf T}.
\end{aligned}
\label{eq:error_propagation}
\end{equation}
For the root, $\widehat{\mathbf E}_0=\widehat{\mathbf L}_0^0(\mathbf G_0^{\mathrm{gt}})^{\mathsf T}$.
Equation~\ref{eq:error_propagation} composes the inherited global error with the current local error conjugated by the true parent orientation $\mathbf G_{p(i)}^{\mathrm{gt}}$.
In particular, $\boldsymbol\Delta_i=\mathbf I$ gives $\widehat{\mathbf E}_i=\mathbf E_{p(i)}$: an accurate local prediction does not remove a residual parent error.
Repeated composition explains accumulation along the tree, but does not imply monotonic error growth, since rotation errors can also partially cancel.

\begin{proof}[Proof of Lemma~\ref{lem:ideal_correction}]
For fixed $\widehat{\mathbf G}_i\in\mathtt{SO}(3)$, the map $\mathbf C\mapsto\mathbf R=\mathbf C\widehat{\mathbf G}_i$ is a bijection on $\mathtt{SO}(3)$, with inverse $\mathbf C=\mathbf R\widehat{\mathbf G}_i^{\mathsf T}$.
\begin{equation}
\begin{aligned}
&\sum_{j\in\mathcal C(i)}\lVert\mathbf C\widehat{\mathbf d}_{ij}-\mathbf d^{\mathrm{gt}}_{ij}\rVert_2^2\\
&\quad=\sum_{j\in\mathcal C(i)}\lVert\mathbf R\mathbf r_{ij}-\mathbf d^{\mathrm{gt}}_{ij}\rVert_2^2,
\qquad \mathbf R=\mathbf C\widehat{\mathbf G}_i.
\end{aligned}
\label{eq:alignment_reparameterization}
\end{equation}
The bijection therefore maps the correction minimizers to the global-rotation minimizers without assuming uniqueness.
The candidate $\mathbf C_0=\mathbf G_i^{\mathrm{gt}}\widehat{\mathbf G}_i^{\mathsf T}$ maps every provisional direction to its exact target, attaining zero cost.
Any other minimizer $\mathbf R$ must consequently satisfy $\mathbf R\mathbf r_{ij}=\mathbf G_i^{\mathrm{gt}}\mathbf r_{ij}$ for all children.
By Appendix~\ref{app:observability_derivation}, $\operatorname{rank}(\mathbf A_i^{\mathrm{loc}})=3$ if and only if at least two rest directions are non-collinear.
In this case, the proper rotation $(\mathbf G_i^{\mathrm{gt}})^{\mathsf T}\mathbf R$ fixes both directions and their cross product.
These vectors span $\mathbb R^3$, so this rotation is identity.
Thus the solution is unique and
\begin{equation}
\mathbf C_i^{\mathrm{ideal}}=\mathbf C_0=\widehat{\mathbf E}_i^{\mathsf T},
\qquad
\mathbf E_i=\mathbf C_i^{\mathrm{ideal}}\widehat{\mathbf E}_i=\mathbf I.
\label{eq:error_cancellation_proof}
\end{equation}
Restoring the frame index, this gives
\begin{equation}
\mathbf C^{\mathrm{ideal}}_{t,i}
=\mathbf G^{\mathrm{gt}}_{t,i}\widehat{\mathbf G}_{t,i}^{\mathsf T},
\qquad
\mathbf G_{t,i}=\mathbf G^{\mathrm{gt}}_{t,i}.
\label{eq:ideal_cancellation}
\end{equation}
Conversely, if $\operatorname{rank}(\mathbf A_i^{\mathrm{loc}})=2$, all outgoing rest directions are collinear.
Composing $\mathbf G_i^{\mathrm{gt}}$ on the left with any rotation about their common ground-truth axis leaves every target direction unchanged, giving multiple optimal global rotations and, by the bijection above, multiple correction minimizers.
If the rank is zero, there are no directional constraints and every rotation minimizes the objective.
Thus the ideal correction is unique if and only if the rank is three.
\end{proof}

This result holds for arbitrary initial rotation errors, not merely small perturbations.
It concerns the aggregate error in the global orientation and does not require recovering the individual ancestor errors that produced it.

\paragraph{What a single axis cannot determine.}
If all rest directions are collinear, write them as $\mathbf r_{ij}=\pm\mathbf r$ and let $\mathbf d^{\mathrm{gt}}=\mathbf G_i^{\mathrm{gt}}\mathbf r$.
The complete set of optimal global rotations is
\begin{equation}
\begin{aligned}
\mathbf G_i(\phi)&=\exp(\phi[\mathbf d^{\mathrm{gt}}]_\times)\mathbf G_i^{\mathrm{gt}},\\
\mathbf E_i(\phi)&=\exp(\phi[\mathbf d^{\mathrm{gt}}]_\times),\quad \phi\in[0,2\pi).
\end{aligned}
\label{eq:twist_ambiguity}
\end{equation}
Here $[\mathbf d^{\mathrm{gt}}]_\times\mathbf x=\mathbf d^{\mathrm{gt}}\times\mathbf x$.
All these rotations align every supported bone, but differ in their twist about the target axis.
Equivalently, the residual minimizers are $\exp(\phi[\mathbf d^{\mathrm{gt}}]_\times)\mathbf C_{\min}$, where $\mathbf C_{\min}$ is a minimum-angle rotation taking $\widehat{\mathbf G}_i\mathbf r$ to $\mathbf d^{\mathrm{gt}}$.
The minimum-angle representative resolves the directional mismatch while leaving the true twist undetermined.
It also provides no guarantee that a twist angle measured in an arbitrary fixed-coordinate decomposition remains unchanged.
For a leaf, the objective is constant, and identity is chosen as the residual correction.

\paragraph{Local rotations, positions, and practical limits.}
Restoring the frame index, the output local rotations are
\begin{equation}
\mathbf L_{t,i}=
\begin{cases}
\mathbf G_{t,0},&i=0,\\
\mathbf G_{t,p(i)}^{\mathsf T}\mathbf G_{t,i},&i>0.
\end{cases}
\label{eq:corrected_local}
\end{equation}
They reproduce the corrected global rotations under forward kinematics.
Recovering $\mathbf L_{t,i}^{\mathrm{gt}}$ requires both the joint and its parent to have correct global orientations.
An exact $\mathbf G_{t,i}$ is therefore insufficient when the corrected parent remains inaccurate.
Likewise, correcting a joint's orientation changes its outgoing bone directions, but does not move that joint or undo an already accumulated parent-position error.
If all bone directions and lengths are exact, forward kinematics recovers root-relative positions even when some axial twists remain ambiguous.
Recovering absolute positions additionally requires the correct root translation.

The cancellation result assumes exact compatible directions and a full ideal update.
Noisy or geometrically incompatible direction predictions, partial gated updates, and numerical regularization do not satisfy these assumptions in general.
The result explains the correction mechanism, not a guarantee of monotonic true-error reduction in the learned system or elimination of temporal drift.

\subsection{Geometric alignment details}
\label{app:alignment}
Figure~\ref{fig:overview}(c) illustrates the single-direction and multi-direction alignment detailed below.
Alignment produces a geometric proposal from provisional directions that incorporate the corrected parent orientation.
The reliability gate then controls how much of this proposal is applied, as described in Section~\ref{sec:cast_gating}.

\paragraph{Closed-form alignment.}
\begin{proof}[Proof of Eq.~\ref{eq:ideal_closed_form}]
Using $\mathbf H^{\mathrm{ideal}}_{t,i}$ and its SVD from Eq.~\ref{eq:ideal_cross_covariance}, the unit-length directions give an objective of $2Z_i-2\operatorname{tr}(\mathbf C(\mathbf H^{\mathrm{ideal}}_{t,i})^{\mathsf T})$, where $Z_i=|\mathcal C(i)|$.
Write $\mathbf Q=\mathbf U_{t,i}^{\mathsf T}\mathbf C\mathbf V_{t,i}$ and $\delta=\det(\mathbf U_{t,i}\mathbf V_{t,i}^{\mathsf T})\in\{-1,1\}$.
Then $\mathbf Q$ is orthogonal with determinant $\delta$, and the trace to maximize is $\operatorname{tr}(\mathbf Q\boldsymbol\Sigma_{t,i})$.
For ordered singular values $\sigma_1\geq\sigma_2\geq\sigma_3\geq0$, each $Q_{kk}\leq1$ gives the upper bound $\sigma_1+\sigma_2+\sigma_3$ when $\delta=1$.
When $\delta=-1$, an orthogonal $3\times3$ matrix satisfies $\operatorname{tr}(\mathbf Q)\leq1$, so
\[
\operatorname{tr}(\mathbf Q\boldsymbol\Sigma_{t,i})
=\sum_{k=1}^3\sigma_k-\sum_{k=1}^3\sigma_k(1-Q_{kk})
\leq\sigma_1+\sigma_2-\sigma_3.
\]
Both bounds are attained by $\mathbf Q=\operatorname{diag}(1,1,\delta)$.
Substituting $\mathbf C=\mathbf U_{t,i}\mathbf Q\mathbf V_{t,i}^{\mathsf T}$ yields Eq.~\ref{eq:ideal_closed_form}, whose determinant is one.
\end{proof}
Uniqueness under compatible directions follows from Lemma~\ref{lem:ideal_correction}.

\paragraph{Observability rank is not cross-covariance rank.}
For $Z_i>0$, under the compatibility assumption of Appendix~\ref{app:error_cancellation},
\begin{equation}
\begin{aligned}
\mathbf B_i&=\sum_{j\in\mathcal C(i)}\mathbf r_{ij}\mathbf r_{ij}^{\mathsf T},\\
\mathbf H^{\mathrm{ideal}}_{t,i}&=\mathbf G^{\mathrm{gt}}_{t,i}\mathbf B_i\widehat{\mathbf G}_{t,i}^{\mathsf T},\\
\mathbf A_i^{\mathrm{loc}}&=\mathbf I-\mathbf B_i/Z_i.
\end{aligned}
\label{eq:alignment_rank_relation}
\end{equation}
Consequently, $\operatorname{rank}(\mathbf H^{\mathrm{ideal}}_{t,i})=\operatorname{rank}(\mathbf B_i)$, whereas $\operatorname{rank}(\mathbf A_i^{\mathrm{loc}})=3$ if and only if $\operatorname{rank}(\mathbf B_i)\geq2$.
In particular, two non-collinear directions uniquely determine a proper rotation even though $\mathbf H^{\mathrm{ideal}}_{t,i}$ then has rank two.
For example, with rest directions $\mathbf e_1,\mathbf e_2$ and identity global rotations, $\mathbf H^{\mathrm{ideal}}=\operatorname{diag}(1,1,0)$ but $\mathbf A^{\mathrm{loc}}=\operatorname{diag}(1/2,1/2,1)$.
Applying the finite-rotation argument in Appendix~\ref{app:error_cancellation} to the initial composition $\widehat{\mathbf G}_{t,i}$ establishes uniqueness without requiring full rank of $\mathbf H^{\mathrm{ideal}}$.
For noisy or incompatible directions, the factorization in Eq.~\ref{eq:alignment_rank_relation} need not hold, and rest-geometry rank alone does not guarantee uniqueness of the alignment optimum.

\paragraph{Minimum-angle correction for collinear support.}
When $\mathbf A_i^{\mathrm{loc}}$ has rank two and the ideal directions are compatible, let $j^\star$ be a selected child bone in $\mathcal C(i)$.
Following Section~\ref{sec:cast}, $\widehat{\mathbf d}_{t,ij}=\widehat{\mathbf G}_{t,i}\mathbf r_{ij}$ denotes the composed direction and $\mathbf d^{\mathrm{gt}}_{t,ij}$ its ground-truth counterpart.
For non-coincident and non-antipodal directions, define
\begin{equation}
\begin{aligned}
\theta_{t,i}&=\mathtt{atan2}\!\left(
\lVert\widehat{\mathbf d}_{t,ij^\star}\times\mathbf d^{\mathrm{gt}}_{t,ij^\star}\rVert_2,\,
\widehat{\mathbf d}_{t,ij^\star}^{\mathsf T}\mathbf d^{\mathrm{gt}}_{t,ij^\star}\right),\\
\mathbf n_{t,i}&=\frac{\widehat{\mathbf d}_{t,ij^\star}\times\mathbf d^{\mathrm{gt}}_{t,ij^\star}}
{\lVert\widehat{\mathbf d}_{t,ij^\star}\times\mathbf d^{\mathrm{gt}}_{t,ij^\star}\rVert_2},\\
\mathbf C_{\min,t,i}&=\exp\!\left(\theta_{t,i}[\mathbf n_{t,i}]_\times\right).
\end{aligned}
\label{eq:rank_two_correction}
\end{equation}
Here $\mathbf n_{t,i}$ is the unit rotation axis, $[\mathbf n_{t,i}]_\times$ is the skew-symmetric matrix satisfying $[\mathbf n_{t,i}]_\times\mathbf x=\mathbf n_{t,i}\times\mathbf x$, and $\exp$ denotes the matrix exponential.
We select $\mathbf C^{\mathrm{ideal}}_{t,i}=\mathbf C_{\min,t,i}$ in this case.
Coincident directions return identity.
For antipodal directions, every $\pi$ rotation about an axis orthogonal to the provisional direction has the minimum angle.
We choose a deterministic axis to select one of these solutions.
This convention selects a representative of the twist-ambiguous solution set in Eq.~\ref{eq:twist_ambiguity}.

\paragraph{Connection to predicted-direction alignment.}
For multiple predicted directions, the confidence-weighted matrix is
\begin{equation}
\mathbf H_{t,i}=\sum_{j\in\mathcal C(i)}c_{t,ij}\mathbf d_{t,ij}\widehat{\mathbf d}_{t,ij}^{\mathsf T}.
\end{equation}
If $\mathbf H_{t,i}=\mathbf U\bm{\Sigma}\mathbf V^{\mathsf T}$, a minimizer of the unregularized weighted objective is
\begin{equation}
\mathbf C^*_{t,i}=\mathbf U\mathtt{diag}(1,1,\mathtt{det}(\mathbf U\mathbf V^{\mathsf T}))\mathbf V^{\mathsf T}.
\end{equation}
The implementation tests the unregularized singular values before selecting multi-direction alignment or a strongest-direction fallback to compute $\mathbf C^*_{t,i}$, which is then gated as in Section~\ref{sec:cast_gating}.
The matrix perturbation used to differentiate the alignment and the swing tolerances are specified in Appendix~\ref{app:numerics}.

\section{Architecture and Implementation Details}

\label{app:implementation}

\subsection{Configured architecture}
\label{app:architecture}
The prediction backbone adapts MoCapAnything~V2~\citep{gong2026mocapanythingv2}.
It constructs topology-conditioned joint queries from the target rest geometry, repeats them over frames, and processes them with image cross-attention, graph attention, and per-joint temporal attention.
The final block's three outputs form the shared features $\mathbf Q$.
A skeleton-conditioned rotation decoder maps these features to initial local rotations, while an auxiliary pose head predicts root-relative positions for training supervision.

The frozen image encoder is DINOv3 ViT-L/16~\citep{simoni2025dinov3}, producing 1024-dimensional patch and CLS tokens from $256\times256$ inputs.
In CAST-B (base), the encoder and rotation decoder each have 8 blocks, width 256, 8 attention heads, and dropout 0.05.
CAST-L (large) uses 16 blocks, width 384, 12 attention heads, and dropout 0.05.
We use $D$ to denote the corresponding feature width.
Both use full-window temporal attention with RoPE.
Their static rest-pose encoders each have six layers.
The rest-pose encoder for the video branch projects an eight-frequency positional encoding of rest positions, then applies graph and joint self-attention with separate feed-forward residuals.
The rest-pose encoder for the rotation branch projects the concatenated three-dimensional offsets and six-dimensional rest rotations and applies graph/feed-forward blocks.
No text-semantic branch, reference-image fusion, pose encoder, or reference-memory encoder is constructed in this configuration.
We normalize rest positions and offsets at the input boundary without supplying the scale as an additional feature.

The rest-pose-conditioned queries have shape $J\times D$ and are repeated over $T$ frames.
In block $\ell$, the outputs after its image, graph, and temporal updates are denoted $\mathbf V^{(\ell)}$, $\mathbf S^{(\ell)}$, and $\mathbf Z^{(\ell)}$, each with shape $T\times J\times D$.
Each graph-attention call forms a dense $J\times J$ score matrix.
The hop matrix stores the shortest-path distance for every pair, while the edge matrix marks self and direct parent--child pairs and assigns a generic relation to the remaining pairs.
With one-based numbering, odd-numbered graph blocks apply the ancestor mask, retaining each query's self token and ancestor chain.
Even-numbered blocks disable this tree mask, allowing attention to every non-padded joint.
The ancestor, descendant, and other-branch labels in the plots therefore describe the static tree relation between a key joint and the query, rather than an additional multi-layer propagation step.
The joint and frame masks exclude padding.
Inputs use the shared rest-size normalization.
The encoder has $N_e$ blocks and the rotation decoder has $N_r$ blocks. Feature concatenation and the rotation-decoder input projection are
\begin{equation}
\begin{aligned}
\mathbf Q&=[\mathbf V^{(N_e)}\Vert\mathbf S^{(N_e)}\Vert\mathbf Z^{(N_e)}]
\in\mathbb R^{T\times J\times3D},\\
\mathbf Y^{(0)}&=\mathbf W_q\mathtt{LN}(\mathbf Q)+\mathbf b_q.
\end{aligned}
\label{eq:base_features}
\end{equation}
Here concatenation is along the feature dimension, and the projection maps each $3D$-dimensional feature to width $D$.
Writing the rest-pose conditioning features for the rotation decoder as $\mathbf b_j$, the decoder computes $\mathbf Y^{(k)}=\mathtt{Dec}_k(\mathbf Y^{(k-1)},\mathbf b;\mathcal S)$ and $\mathbf u=\mathtt{Head}_{\mathrm{rot}}(\mathbf Y^{(N_r)})$.

Each rotation-decoder block applies FiLM conditioning~\citep{perez2018film}, per-joint temporal self-attention, an FFN, graph attention within each frame, and a second FFN, in that order.
Its graph attention follows the same alternating ancestor-mask schedule: with one-based numbering, an odd-numbered block can assign nonzero weight only to the query itself or an ancestor, whereas an even-numbered block directly attends to all non-padded joints.
Both attention--FFN pairs use pre-normalization and residual connections.
Reference-memory cross-attention is disabled in this configuration.
The output head predicts a 6D residual $\mathbf u_{t,j}$~\citep{zhou2019rot6d}, composed with the rest-pose rotation as
\begin{equation}
\widehat{\mathbf L}^{0}_{t,j}=\mathbf L^{\mathrm{rest}}_j\mathtt{GS}(\mathbf u_{t,j}),
\label{eq:base_rotation}
\end{equation}
where $\mathtt{GS}$ denotes Gram--Schmidt rotation decoding.

The rotation interface is $\mathtt{LN}$ followed by a $3D\to D$ linear projection.
The auxiliary position head is $\mathtt{LN}$, $3D\to128$ linear, $\mathtt{GELU}$, and $128\to3$ linear, followed by root subtraction.
Its gradients are not detached from the video features.
The base output assigns rest rotations to dataset-marked static joints.
Because the configured GRACE path does not reapply this static mask, it can still correct static-marked non-leaf joints.
Leaves retain their base local rotations.

The direction head described in Section~\ref{sec:cast_gating} applies $\mathtt{LN}$ to the child and parent features, followed by separate $D\to D$ projections.
For each bone, the normalized length is $\ell_{ij}=\lVert\mathbf o_j\rVert_2/\bar\ell_{\mathcal S}$, where $\bar\ell_{\mathcal S}$ is the mean rest-bone length of skeleton $\mathcal S$.
The four-dimensional rest descriptor $[\mathbf r_{ij},\mathtt{log}(1+\ell_{ij})]$ passes through a $4\to D\to D$ $\mathtt{MLP}$ with $\mathtt{SiLU}$.
The three results are added, then processed by $\mathtt{LN}$, $\mathtt{SiLU}$, a $D\to D$ linear layer, and $\mathtt{SiLU}$.
Separate linear heads predict a direction residual and a reliability logit.
The direction branch adds its residual to the unit rest direction and normalizes the sum to unit length.
The reliability branch maps its logit to confidence through a $\mathtt{sigmoid}$.
Specifically, $\mathbf d_{t,ij}=\mathtt{normalize}(\mathbf r_{ij}+\boldsymbol\eta_{t,ij})$, where $\boldsymbol\eta_{t,ij}\in\mathbb R^3$ is the learned residual in the three output channels.
This parameterization uses the local components $\mathbf r_{ij}$ as a fixed numerical bias, without transforming them from parent-local to camera coordinates during the addition.
The direction-residual head is zero-initialized, and the reliability bias is $-4$.
Direction normalization falls back to the rest direction when the predicted residual cancels it, i.e. when $\lVert\mathbf r_{ij}+\boldsymbol\eta_{t,ij}\rVert_2<\varepsilon_o$ (threshold in Appendix~\ref{app:numerics}).

\subsection{Coordinate conventions}
\label{app:coordinates}
The training cache expresses joint positions and root rotations in the OpenCV camera coordinate system, while non-root rotations are relative to their parents.
Position supervision is root-centered, and bone direction targets are obtained by normalizing parent--child position differences in camera coordinates.
The Bone Direction Head predicts directions in camera coordinates directly from video features.
Since the predicted root rotation is already expressed in camera coordinates, composition along the kinematic tree yields global orientations and bone directions in camera coordinates.
GRACE computes each provisional direction as $\widehat{\mathbf d}_{t,ij}=\widehat{\mathbf G}_{t,i}\mathbf r_{ij}$ and aligns it directly with the predicted direction, without applying a world-to-camera transform.
Camera extrinsics retained with the training data are metadata and are excluded from the model inputs.
Unlabeled-video inference requires no camera intrinsics or extrinsics.
It outputs root rotations in camera coordinates and parent-local non-root rotations, from which FK produces root-centered positions in camera coordinates.
Here, \emph{global} refers to the orientation obtained by composition along the skeleton hierarchy, rather than an orientation in world coordinates.
Alignment requires a common coordinate system for $\widehat{\mathbf d}_{t,ij}$ and $\mathbf d_{t,ij}$.
The geometric observability and ideal error-cancellation results hold independently of the choice of that coordinate system.

\subsection{Numerical implementation}
\label{app:numerics}
Direction normalization uses $\varepsilon_o=10^{-6}$.
For the numerical thresholds below, let $\varepsilon=10^{-6}$ and define the number of outgoing confidences above this threshold as $N_{t,i}=\sum_{j\in\mathcal C(i)}\mathbf 1[c_{t,ij}>\varepsilon]$.
The aggregated outgoing reliability $\bar c_{t,i}$ and confidence-weighted angular disagreement $\delta_{t,i}$ are
\begingroup
\allowdisplaybreaks[4]
\begin{align}
\bar c_{t,i}&=\mathtt{clip}\!\left(\frac{\sum_{j\in\mathcal C(i)}c_{t,ij}}{\max(1,N_{t,i})},0,1\right),\nonumber\\
\delta_{t,i}&=\frac{\sum_{j\in\mathcal C(i)}c_{t,ij}\vartheta_{t,ij}}
{\max(\varepsilon,\sum_{j\in\mathcal C(i)}c_{t,ij})},\label{eq:gate_statistics}\\
\vartheta_{t,ij}&=\mathtt{atan2}\!\left(\lVert\widehat{\mathbf d}_{t,ij}\times\mathbf d_{t,ij}\rVert_2,
\widehat{\mathbf d}_{t,ij}^{\mathsf T}\mathbf d_{t,ij}\right).\nonumber
\end{align}
\endgroup
The gate uses the child count and the spectrum of the rest-direction matrix as its geometric descriptor $\mathbf a_i$.
A sigmoid $\mathtt{MLP}$ takes the joint features $\mathbf Q_{t,i}$, reliability $\bar c_{t,i}$, disagreement $\delta_{t,i}$, and descriptor $\mathbf a_i$ to predict a learned gate factor $\gamma_{t,i}\in[0,1]$.
The update strength in Eq.~\ref{eq:lie_gate} is $g_{t,i}=\bar c_{t,i}\gamma_{t,i}$.
The gate $\mathtt{MLP}$ has two 256-wide $\mathtt{SiLU}$ hidden layers and output bias initialized to $-4$.
This initialization limits early updates while the direction head is being learned.
Multi-direction alignment requires at least two weights above $\varepsilon$, $\sigma_1>\varepsilon$, and $\sigma_2>10^{-4}\sigma_1$ for the unregularized alignment matrix.
Otherwise it uses the strongest-direction swing, or identity when total weight is at most $\varepsilon$.

Geometric correction operations are evaluated in float32.
Proposal selection uses child count and the confidence-weighted alignment matrix, not a separate rank-two/rank-three hard mask.
Gradients through gate confidence and disagreement are enabled.
The singular-value check uses detached inputs to select the numerical proposal branch.
The corrected-motion path remains differentiable.

Multi-child conditioning uses the singular values of the detached, unregularized $\mathbf H$.
If the relative-singular-value test fails, the solver selects the strongest-direction swing without forcing the gate off.
The minimum-swing implementation treats $\lVert\widehat{\mathbf d}\times\mathbf d\rVert^2\leq10^{-6}$ as parallel: it returns identity for nonnegative dot product and a $\pi$ rotation about a deterministic orthogonal axis otherwise.
Thus the ideal shortest-rotation rule has a small numerical angular tolerance.

In Eq.~\ref{eq:lie_gate}, $\mathtt{log}$ denotes a selected real skew-symmetric rotation logarithm with angle in $[0,\pi]$ and $\mathtt{log}(\mathbf I)=\mathbf 0$.
Matrix proposals are converted to axis--angle form through a unit quaternion with nonnegative scalar part.
At exactly $\pi$, the conversion chooses the axis sign so that its largest-magnitude component is positive, resolving ties in $x,y,z$ order.
For an antipodal swing, we choose the coordinate basis vector least aligned with the source direction, resolving ties in $x,y,z$ order.
We then project it onto the plane perpendicular to the source direction and normalize it to obtain the rotation axis.
The direct axis--angle swing path retains this selected axis when applying the gate.
In either path, gating scales the selected rotation vector by $g_{t,i}$ before exponentiation, fixing the partial update at the $\pi$ boundary.

For the SVD used to construct the rotation, the implementation applies
\begin{equation}
\mathbf H_{\mathrm{num}}=\mathbf H+
4\varepsilon\max(\lVert\mathbf H\rVert_F,\varepsilon)
\mathtt{diag}(1,2,4),
\qquad \varepsilon=10^{-6}.
\end{equation}
This small perturbation reduces repeated-singular-value instability during differentiation.
It does not enter the observability test.
Determinant correction ensures that $\mathbf C^*_{t,i}$ is a proper rotation.
Degenerate proposals use the strongest-direction swing or identity.

Same-depth joints are processed together, with a sequential loop only over tree depth.
Topology preparation currently constructs dense joint-pair metadata ($O(J^2)$ storage per rig).
Per-frame direction processing uses $O(JK_{\max})$ padded child slots, where $K_{\max}$ is the maximum retained child count in the batch.
The batched implementation constructs $3\times3$ SVD candidates even for single-child rows before selecting swing or multi-child alignment.
Accordingly, this implementation is not SVD-free, and its padded runtime should not be equated with an unpadded $O(|\mathcal E|)$ bound.
The count feature's saturation at four does not truncate children used for alignment.
The gate $\mathtt{MLP}$ is evaluated in chunks of at most 4,096 tokens.

\subsection{Training penalties and reductions}
\label{app:training}
The auxiliary pose and rotation objectives in Eq.~\ref{eq:objective} are
\begin{equation}
\mathcal L_{\mathrm{pose}}=
\lambda_{\mathrm{pos}}\mathcal L_{\mathrm{pos}}
+\lambda_{\mathrm{vel}}\mathcal L_{\mathrm{vel}},
\label{eq:pose_objective}
\end{equation}
\begin{equation}
\mathcal L_{\mathrm{rotation}}=
\lambda_{\mathrm{rot}}\mathcal L_{\mathrm{rot}}
+\lambda_{\mathrm{rotv}}\mathcal L_{\mathrm{rot\text{-}vel}}
+\lambda_{\mathrm{root}}\mathcal L_{\mathrm{root}}.
\label{eq:rotation_objective}
\end{equation}
The auxiliary terms $\mathcal L_{\mathrm{pos}}$ and $\mathcal L_{\mathrm{vel}}$ supervise joint positions and adjacent-frame displacements.
The rotation terms $\mathcal L_{\mathrm{rot}}$, $\mathcal L_{\mathrm{rot\text{-}vel}}$, and $\mathcal L_{\mathrm{root}}$ supervise corrected rotations, their adjacent-frame changes, and root orientation, respectively.
The implementation uses the loss weights
$(\lambda_{\mathrm{pos}},\lambda_{\mathrm{vel}},\lambda_{\mathrm{rot}},\lambda_{\mathrm{rotv}},\lambda_{\mathrm{FK}},\lambda_{\mathrm{bone}},\lambda_{\mathrm{root}})=(6,80,1,5,6,1,0.1)$.

The expectations in Eqs.~\ref{eq:direction_loss} and~\ref{eq:reliability_target} cover frame--edge entries, with per-sample reductions specified below.
The fixed decay scale in Eq.~\ref{eq:reliability_target} is $\tau=\pi/9$ radians: errors of $20^\circ$ and $40^\circ$ give reliability targets $e^{-1}$ and $e^{-2}$.
This is a supervision scale, not a probability of correctness.
Empty directional support contributes zero.

The position, adjacent-frame displacement, and FK losses use coordinate-wise $\mathtt{SmoothL1}$ with transition $\beta=0.075$ in the training-normalized coordinates, averaged over coordinates and valid non-root joints/frames.
FK targets are annotated root-relative positions rather than positions regenerated from target rotations.
All rotation penalties use $\mathtt{SmoothL1}$ with transition one radian on the geodesic angle, evaluated with a $\mathtt{atan2}$ formula.
Rotation-change loss compares $\mathbf L_{t,j}\mathbf L_{t-1,j}^{\mathsf T}$ with its annotated counterpart.
Rotation supervision applies to the final corrected rotations, with no separate loss on the uncorrected base predictions in this configuration.

The per-frame rotation loss is motion-weighted.
Let $v^{\mathrm{gt}}_{t,j}=\mathtt{dist}_{\mathtt{SO}(3)}(\mathbf L^{\mathrm{gt}}_{t,j},\mathbf L^{\mathrm{gt}}_{t-1,j})$ for valid adjacent frames.
A frame's detached weight is the maximum of its two adjacent speeds, with invalid pairs contributing zero.
At a sequence endpoint, we use the sole available adjacent speed.
The weighted error is normalized by the sum of weights within each sample.
The root-orientation term $\lambda_{\mathrm{root}}\mathcal L_{\mathrm{root}}$ in Eq.~\ref{eq:rotation_objective} uses uniform validity weights.
Rotation losses include the root and valid static-marked joints.
Entries with zero speed have zero weight in the motion-weighted term.
All losses are averaged equally across samples with nonempty support (positive total weight for weighted terms), with zero returned for empty support.
Direction and reliability losses use the same per-sample reduction without motion weighting.
The $1^\circ$ directional-motion threshold applies only to diagnostics and does not mask these losses.
The $\mathtt{SmoothL1}$ transition for reliability is one.

Training uses 81-frame clips, AdamW with learning rate $3\times10^{-4}$, zero weight decay, a 5\% warmup, cosine decay to $3\times10^{-6}$, a 6,000-step budget, seed 42, bfloat16 mixed precision, and gradient-norm clipping at one.

\subsection{Inference implementation and optimization}
\label{app:inference_optimization}
\paragraph{Prediction heads and geometric correction.}
Inference skips the auxiliary position head used for training supervision while retaining the bone-direction and reliability predictions required by GRACE.
For a fixed target skeleton, the optimized CUDA path processes joints by tree depth, fusing parent-rotation composition, directional alignment, learned gating, rotation blending, and local/global rotation writeback into one kernel launch per depth.
Joints at the same depth are processed in parallel, while successive depths preserve the dependency on corrected parent orientations.
Cached topology indices restrict each launch to the current depth, avoiding repeated processing of the full joint set and reducing kernel-launch overhead.
Supported fixed-topology inference configurations use this fused path, with the general implementation available as a fallback.

\paragraph{Attention backends.}
The implementation uses the variable-length FlashAttention interface for supported image cross-attention and full-window temporal-attention calls on CUDA with half-precision or bf16 tensors.
Valid tokens are packed to exclude padding, and standard PyTorch attention provides a fallback when this path is unavailable.
Graph attention uses PyTorch scaled dot-product attention (SDPA) with its learned relation biases and topology masks.
Temporal attention for chunk sizes greater than one also uses SDPA with an explicit chunk-causal mask, allowing attention within the current chunk and to preceding chunks.
Thus, the attention backend depends on the operator and its mask rather than using FlashAttention uniformly throughout the model.

\paragraph{Streaming execution.}
Single-frame causal inference maintains a separate temporal key--value cache for each encoder and decoder layer.
Only the incoming frame's projections and rotary embeddings are computed, and SDPA attends to the cached keys and values.
This cache path handles single-frame chunks, while larger chunks use the separate chunk-causal execution path.
The streaming implementation warms up the cache and compilation before using the compiled steady-state path, avoiding recompilation as the cache initially fills.
The latency measurements in Table~\ref{tab:efficiency} include the frozen image encoder and model computation, but exclude video acquisition and mesh rendering.

\section{Blender Add-on}
\label{app:blender_addon}

For interactive use we drive a rigged character in Blender from the causal variant through an add-on that talks to a running CAST service.
The add-on imports the target character as an FBX armature and receives the posed skeleton of the same joint count over the network, so the incoming motion maps onto the armature without a separate retargeting step and the rig follows the performance live in the viewport.
Its panel groups the runtime, connection, and transfer controls together with a status log that reports the imported armature and every skeleton it receives; the deployment in our setup is configured for 24 FPS input and uses a pose port with a separate control port (Figure~\ref{fig:blender_addon}(a)).
Figure~\ref{fig:blender_addon}(b) shows a live capture session, where the performer is filmed while the add-on drives the character in real time, so a take can be inspected while it is being recorded.

\begin{figure*}[!t]
\centering
\includegraphics[width=\linewidth]{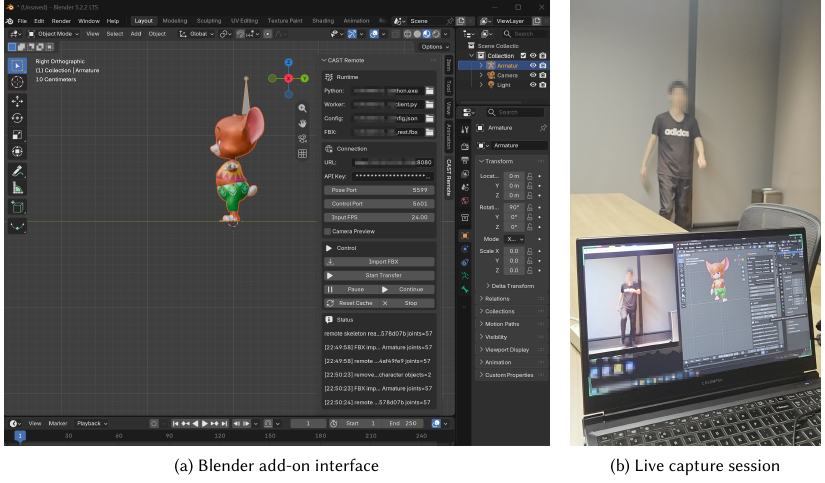}
\caption{\textbf{The Blender add-on for live capture.} (a) The add-on panel: runtime paths, service URL and API key, pose and control ports, transfer controls, and a status log. (b) A live capture session: the performer is recorded while the add-on drives the rigged character in real time.}
\label{fig:blender_addon}
\end{figure*}

\section{Evaluation Details}

\label{app:protocol}

\subsection{Data and split sizes}
\label{app:data}
Table~\ref{tab:data_scale} lists the processed collections used in this work.
A \emph{rig record} is a preprocessing identifier, not a graph-isomorphism class: asset variants and coordinate canonicalization can produce multiple records for the same structural pattern.
Motion counts refer to underlying animation clips.
Following the protocol in the main text, we split Zoo's 1,163 motions into 95\% for training and 5\% for testing.
For MObjaverse, we preserve the Train, Seen, and Unseen split assignments released with TopoCap, whose original lists contain 17,976, 469, and 469 clips, respectively.
During preprocessing, we exclude 21 training clips, one Seen clip, and one Unseen clip whose rendered meshes exhibit invalid geometry.
The retained splits therefore contain 17,955, 468, and 468 clips (18,891 in total), and this same quality-filtered subset is used for all reported MObj evaluations.
The Zoo row describes the collection before splitting, while the MObj rows describe these retained subsets of the official splits.
Matched-rig results are reported in Table~\ref{tab:combined}.

\begin{table*}[!t]
\caption{Statistics of the processed collections.
Zoo statistics describe the collection before the 95\%/5\% training/test split.
MObj statistics follow the official Train, Seen, and Unseen assignments after rendering-quality filtering.}
\label{tab:data_scale}
\begin{center}
\begin{tabular*}{\linewidth}{@{\extracolsep{\fill}}lccc@{}}\toprule
Collection/split & Rig records & Motion clips & Frame records\\\midrule
Zoo & 137 & 1,163 & 1,396,500\\
MObj Train & 12,723 & 17,955 & 1,905,136\\
MObj Seen & 440 & 468 & 39,261\\
MObj Unseen & 468 & 468 & 64,279\\\bottomrule
\end{tabular*}
\end{center}
\end{table*}

Truebones Zoo is the original source of the animal assets~\citep{truebones2022}.
Its assets remain subject to the applicable vendor license and are excluded from the planned release of our processed data.
MObjaverse is derived from Objaverse-XL and introduced with TopoCap~\citep{pu2026topocap,deitke2023objaversexl}.
Our preprocessing records and quality-filtered splits are distinguished from the original datasets so that their counts are not interpreted as new source-dataset statistics.

\subsection{Metric definitions}
\label{app:metrics}
Let $\mathbf x_{t,j}$ and $\mathbf x^{\mathrm{gt}}_{t,j}$ denote predicted and ground-truth joint positions in a common coordinate system, for a sequence with $T$ frames.
For the correspondence-based metrics, $J$ denotes the number of corresponding joints, indexed from $0$ to $J-1$ as in Section~\ref{sec:setup}.

\paragraph{Mean per-joint position error (MPJPE).}
MPJPE measures the average Euclidean distance between corresponding predicted and ground-truth joint positions:
\begin{equation}
\mathrm{MPJPE}
=\frac{1}{TJ}\sum_{t=1}^{T}\sum_{j=0}^{J-1}
\lVert\mathbf x_{t,j}-\mathbf x^{\mathrm{gt}}_{t,j}\rVert_2.
\label{eq:metric_mpjpe}
\end{equation}

\paragraph{Procrustes-aligned MPJPE (PA-MPJPE).}
PA-MPJPE measures joint-position error after aligning each predicted pose to its reference by a similarity transformation.
For each frame, the alignment minimizes the sum of squared distances over corresponding joints:
\begin{equation}
(a_t^*,\mathbf R_t^*,\mathbf b_t^*)
=\operatorname*{arg\,min}_{a\geq0,\,\mathbf R\in\mathrm{SO}(3),\,\mathbf b\in\mathbb R^3}
\sum_{j=0}^{J-1}\lVert a\mathbf R\mathbf x_{t,j}+\mathbf b-\mathbf x^{\mathrm{gt}}_{t,j}\rVert_2^2.
\end{equation}
With $\widetilde{\mathbf x}_{t,j}=a_t^*\mathbf R_t^*\mathbf x_{t,j}+\mathbf b_t^*$, the metric is
\begin{equation}
\mathrm{PA\text{-}MPJPE}
=\frac{1}{TJ}\sum_{t=1}^{T}\sum_{j=0}^{J-1}
\lVert\widetilde{\mathbf x}_{t,j}-\mathbf x^{\mathrm{gt}}_{t,j}\rVert_2.
\label{eq:metric_pa_mpjpe}
\end{equation}
The alignment removes differences in translation, rotation, and uniform scale without allowing reflections.

\paragraph{Mean per-joint velocity error (MPJVE).}
MPJVE measures the discrepancy between predicted and ground-truth joint velocities.
For a sampling interval $\Delta\tau$, define
$\mathbf v_{t,j}=(\mathbf x_{t,j}-\mathbf x_{t-1,j})/\Delta\tau$, and define $\mathbf v^{\mathrm{gt}}_{t,j}$ analogously.
For $T\geq2$,
\begin{equation}
\mathrm{MPJVE}
=\frac{1}{(T-1)J}\sum_{t=2}^{T}\sum_{j=0}^{J-1}
\lVert\mathbf v_{t,j}-\mathbf v^{\mathrm{gt}}_{t,j}\rVert_2.
\label{eq:metric_mpjve}
\end{equation}

\paragraph{Chamfer distance (CD).}
CD measures geometric agreement between the predicted and ground-truth skeletal point clouds without requiring joint correspondence.
For point sets $\mathcal P$ and $\mathcal Q$, we use the symmetric, unsquared Euclidean form:
\begin{equation}
\begin{aligned}
\operatorname{CD}(\mathcal P,\mathcal Q)
&=\frac{1}{2}\Biggl(
\frac{1}{|\mathcal P|}\sum_{\mathbf x\in\mathcal P}\min_{\mathbf y\in\mathcal Q}\lVert\mathbf x-\mathbf y\rVert_2\\
&\qquad+\frac{1}{|\mathcal Q|}\sum_{\mathbf y\in\mathcal Q}\min_{\mathbf x\in\mathcal P}\lVert\mathbf y-\mathbf x\rVert_2
\Biggr).
\end{aligned}
\label{eq:metric_chamfer}
\end{equation}
The two terms measure nearest-neighbor distances in opposite directions, and the factor $1/2$ averages them.
For each frame, $\mathcal P_t$ and $\mathcal Q_t$ contain the predicted and reference joint positions, respectively.
Averaging over frames gives the sequence score $T^{-1}\sum_{t=1}^{T}\operatorname{CD}(\mathcal P_t,\mathcal Q_t)$.

\subsection{Evaluation protocol}
\label{app:evaluation_protocol}
We evaluate two tasks using the metrics defined in Appendix~\ref{app:metrics}.

\paragraph{Matched-rig motion capture.}
The input video depicts motion on the same skeleton that is supplied to the model as the target rig.
We evaluate on the held-out 5\% of Zoo and the official MObjaverse Seen and Unseen splits, following the split assignments described in Appendix~\ref{app:data}.
The annotated motion that generated each input video serves as the reference, with predicted and reference joints matched on the supplied rig.
For CAST, predicted joint positions are obtained by forward kinematics from the predicted rotations and the target rest skeleton.

\paragraph{Cross-skeleton animation.}
The target skeleton differs from the source character's skeleton in the input video.
Our curated Mixamo test set contains 60 animated sequences, with ten animations for each of six character skeletons.
Each selected source sequence is paired with the other five target skeletons, producing 300 ordered source--target pairs.
The fixed manifest identifies the source motion and target rig for each pair.
For each pair, we download the corresponding animation for the target character from Mixamo, which supplies the motion through its own retargeting process.
This target-rig animation serves as the reference ground truth.
The prediction and reference therefore share the target skeleton and joint correspondence, even though the source and target skeletons differ.
The reference motions are used only for evaluation and are not provided to CAST at inference in either task.

\paragraph{Normalization and units.}
We follow TopoCap's unit-bounding-box normalization and distance-reporting convention~\citep{pu2026topocap}.
Let $L$ be the longest side of the target rest skeleton's axis-aligned bounding box before normalization.
The coordinates $\mathbf x$ used in Appendix~\ref{app:metrics} are root-centered positions divided by $L$, with the same fixed scale applied to the prediction and reference throughout the sequence.
Equivalently, they are one half of the coordinates whose rest bounding box has longest side $2$.
We multiply MPJPE and PA-MPJPE by $1000$ to follow the benchmark's millimeter reporting convention.
Applying the same length conversion to MPJVE gives millimeters per frame.
CD is reported directly in the normalized evaluation coordinates.
For MPJVE, we use $\Delta\tau=1$ frame, so the evaluated velocities are adjacent-frame joint displacements.
All four metrics include the root and all other valid joints and exclude padded joints and frames.
Only PA-MPJPE applies a per-frame similarity alignment.

\paragraph{Full-sequence inference and aggregation.}
For our model, temporal windows have length 81 and are non-overlapping during evaluation.
The predictions from consecutive windows are concatenated to form the complete animation.
MCA V2$^{\ast}$ uses its public 32-frame windows, TopoCap processes 8 frames per pass, and Puppeteer is run in 32-frame chunks.
For both tasks, all metrics are computed over complete animations.
Temporal differences include adjacent frames on opposite sides of an inference-window boundary.
Aggregation is by animation rather than by window or chunk: however a sequence is divided at inference time, it contributes exactly one weight, and dataset scores are arithmetic means of these per-animation scores for every method.

\subsection{User-study protocol and questionnaire}
\label{app:user_study}
\paragraph{Interface and stimuli.}
We conduct an anonymous web-based ranking study comparing CAST-L, MCA V2$^{\ast}$, and TopoCap.
The questionnaire supports Chinese and English; its interface is shown in Figure~\ref{fig:user_study_interface}.
Each item presents one input or source animation and three anonymized method outputs labeled A, B, and C.
The method names are not shown during ranking.
The study contains 20 items per participant: ten matched-rig motion-capture items, including five Zoo and five MObjaverse items, and ten cross-skeleton items, including three Zoo, three MObjaverse, and four Mixamo items.
The matched-rig block is presented first, followed by the cross-skeleton block.
We randomize the item order within each block for every participant and independently randomize the assignment of methods to A/B/C for each item.

\begin{figure*}[!t]
\centering
\includegraphics[width=\linewidth,height=0.72\textheight,keepaspectratio]{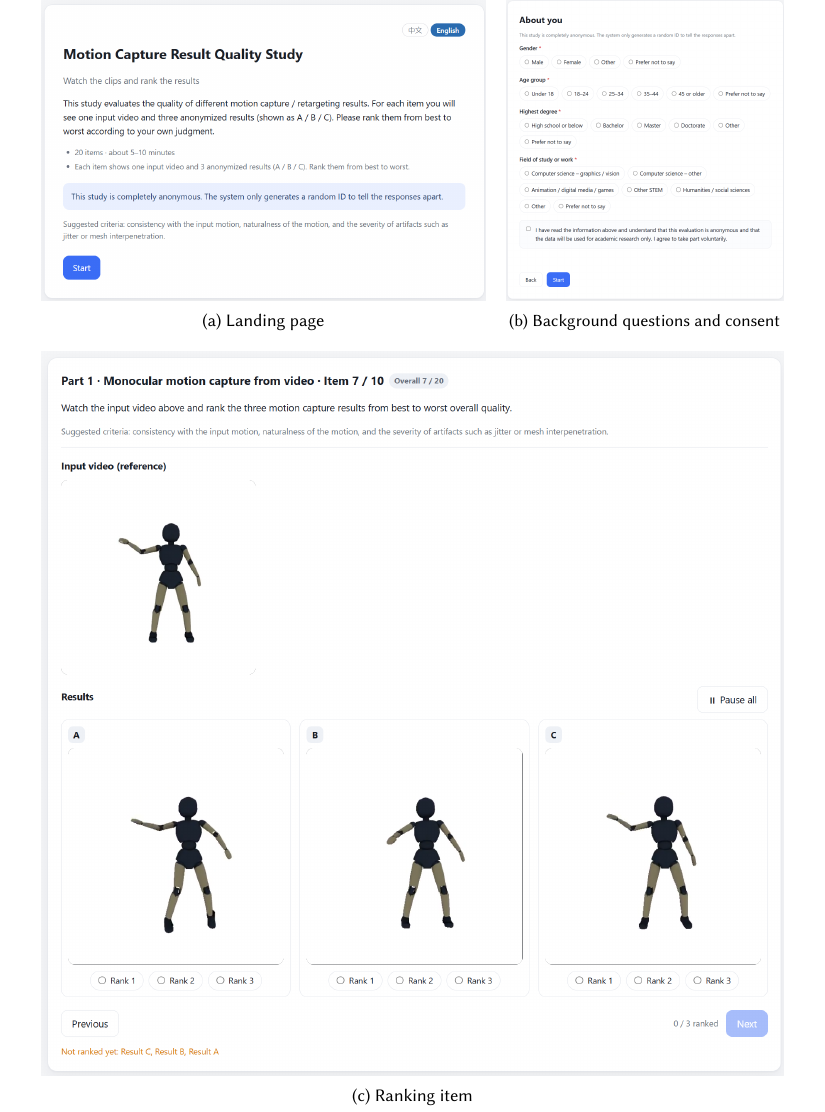}
\caption{\textbf{The ranking interface used for the user study.} (a) The landing page, which states the task, the estimated duration, and the anonymity of the study. (b) The background questions and the consent box. (c) A ranking item: one input or source animation above three anonymized results, each assigned a distinct rank.}
\label{fig:user_study_interface}
\end{figure*}

\paragraph{Ranking questions and instructions.}
The English interface uses the following task-specific prompts:
\begin{enumerate}
    \item \textbf{Matched-rig motion capture:} ``Watch the input video above and rank the three motion capture results from best to worst overall quality.''
    \item \textbf{Cross-skeleton animation:} ``Using the source skeletal animation, rank the three cross-skeleton results from best to worst overall quality.''
\end{enumerate}
Participants are asked to consider consistency with the input motion, motion naturalness, and the severity of artifacts such as jitter or mesh interpenetration.
Participants use these dimensions to form a single overall-quality ranking without assigning separate scores to each dimension.
Participants assign ranks 1 (best), 2, and 3 (worst), using each rank exactly once.
The interface requires a complete ranking without ties before proceeding.
Videos are muted and looped, with individual playback controls and a shared play/pause control for the three candidate outputs.
Participants may replay the videos and revisit earlier items before submitting their answers.

\paragraph{Background questions and consent.}
Before ranking, participants answer four single-choice background questions.
Each requires a selection but includes a ``Prefer not to say'' option:
\begin{itemize}
    \item \textbf{Gender:} Male, Female, Other, or Prefer not to say.
    \item \textbf{Age group:} Under 18, 18--24, 25--34, 35--44, 45 or older, or Prefer not to say.
    \item \textbf{Highest degree:} High school or below, Bachelor, Master, Doctorate, Other, or Prefer not to say.
    \item \textbf{Field of study or work:} Computer science--\allowbreak graphics/\allowbreak vision, Computer science--\allowbreak other, Animation/\allowbreak digital media/\allowbreak games, Other STEM, Humanities/social sciences, Other, or Prefer not to say.
\end{itemize}
The interface states that participation is voluntary, the evaluation is anonymous, and the data are used only for academic research.
Participants must actively select an initially unchecked consent box before starting the ranking tasks.

\paragraph{Anonymity and response recording.}
No account, name, or email address is requested, and the study application does not record IP addresses in the submitted response records.
A randomly generated participant ID associates the answers within one questionnaire and prevents duplicate storage of the same submission.
The browser stores progress locally to support resuming the questionnaire.
Submitted records contain the rankings, background-question selections, consent status, presentation order and randomization seed, interface language, timing information, and browser user-agent string.

\paragraph{Participants and aggregation.}
Eighteen participants completed all 20 items, producing 360 rankings, with a mean completion time of 580 seconds.
The recorded sample includes 13 male and 5 female participants.
Eight participants are aged 18--24, eight are aged 25--34, one is aged 35--44, and one preferred not to report their age.
The highest reported degrees are doctorate ($n=6$), master's ($n=6$), and bachelor's ($n=5$), with one participant preferring not to say.
Nine participants report computer graphics or vision, six report other computer-science fields, and three prefer not to say.
We pool all 360 rankings, with each task contributing 180 rankings from 18 participants evaluating ten items.
First-choice rate is the fraction of rankings that place a method first, and mean rank is its average assigned rank, with lower values indicating stronger preference.
These are descriptive summaries of repeated judgments by the same 18 participants.

\begin{figure*}[!t]
\centering
\includegraphics[width=\linewidth]{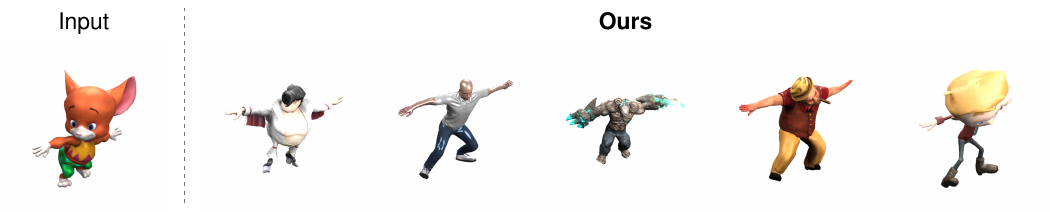}
\caption{\textbf{Twist errors.} The input is shown to the left of the vertical divider; CAST (Ours) outputs on five target characters are shown to the right. The overall body configuration is largely preserved, but twist errors cause the head orientation to differ from the input. With rank-two directional support, outgoing bone directions constrain swing but do not determine axial twist.}
\label{fig:limitation_twist}
\end{figure*}

\section{Additional Model Analysis}
\label{app:model_analysis}

\subsection{Graph-attention analysis protocol}
\label{app:graph_attention}
Figure~\ref{fig:graph_attention} supports the whole-model analysis in Section~\ref{sec:main_results} by comparing CAST-B trained with and without the complete GRACE configuration, alongside MCA V2$^{\ast}$ as a reference.
We evaluate graph attention on the Zoo dataset, using a deterministic selection with seed 42, fixed before inspecting attention, and the 0-degree view for every selected clip.
The selection contains 126 training and 11 test sequences, so this analysis characterizes learned attention across the skeleton collection rather than held-out generalization alone.
All three models receive the same selected clips and valid joint queries.
The analysis covers 16{,}123 valid frames and 860{,}095 frame--query observations per model in 268 non-overlapping 81-frame windows, excluding tail padding.
All models use the same windows for this attention comparison, with MCA V2$^{\ast}$ evaluated using its re-trained checkpoint.
Its quantitative motion evaluation retains the window setting in Appendix~\ref{app:evaluation_protocol}.
Both CAST-B variants use frozen DINOv3 features, eight encoder and eight decoder blocks, width 256, eight attention heads, and checkpoints trained for 6{,}000 steps with seed 42 and the same data split.
The GRACE configuration enables the direction head, correction module, and associated direction losses together.
We compare separately trained configurations, so the results reflect the combined contribution of these components rather than an inference-time toggle of correction alone.

We capture softmax graph attention in the final encoder layer, layer 8 with one-based numbering (code index 7).
This layer allows attention to all valid joints, so locality is not imposed by an ancestor mask.
We average over the eight heads and then over keys within each nonempty topology group for each valid frame--query pair, excluding the query itself without renormalizing the remaining weights.
The groups are direct parent, direct child, grandparent, grandchild, siblings, joints at shortest-path distance exactly three on the undirected skeleton graph, and all other non-self joints.
Empty groups are omitted.
Figure~\ref{fig:graph_attention} pools these frame--query values into boxplots without averaging per skeleton.
Boxes show interquartile ranges, with whiskers spanning the 5th--95th percentiles.
Circles indicate means and white lines indicate medians, with outliers omitted.
Recording attention leaves the native forward outputs unchanged in capture-equivalence checks for all three models.

In this layer, both CAST-B variants emphasize direct parents and children.
With GRACE, mean parent attention increases and mean attention to each more distant relation group decreases, while mean child attention is lower but remains prominent.
These observations characterize the final encoder aggregation and do not establish the same attention pattern in every layer or a causal effect of correction alone.

\subsection{Scope of component evidence}
\label{app:ablations}
\paragraph{Reported comparisons.}
Table~\ref{tab:matched_ablation} reports CAST-B comparisons with and without GRACE for DINOv2 and DINOv3.
The consistent reductions in MPJPE and PA-MPJPE support the contribution of the complete GRACE configuration under both visual encoders.
They do not separately identify the effects of its direction supervision, alignment proposal, reliability gate, or corrected-parent recursion.

\paragraph{Relation to the ideal analysis.}
The ideal cancellation result concerns ancestor-composed global orientations, whereas the reported positional metrics aggregate the final FK output over joints.
Lower positional error therefore does not by itself establish that correction attenuates global-rotation error at every tree depth or in every observability class.
We do not report component-isolating or depth-stratified error-propagation experiments here.

\paragraph{Backbone controls.}
Attention order, final-block feature fusion, and model scaling require separate capacity-controlled comparisons before the full-system gains can be assigned to them.
In particular, matching block counts or widths alone does not establish equal parameter counts, and raw attention weights do not substitute for such controls.

\section{Limitations}
\label{sec:limitations}

GRACE can correct only the rotational components constrained by outgoing bone directions.
Leaves provide no such constraints, and collinear outgoing bones leave axial twist undetermined, so these degrees of freedom depend on the learned rotation predictor.
In particular, a leaf's rotation does not affect skeletal joint positions but can change the deformation of its attached mesh.
Low joint-position errors therefore do not ensure accurate surface deformation.
Additional surface or orientation cues would be needed to constrain these ambiguities geometrically.

Figure~\ref{fig:limitation_twist} illustrates twist errors: the target characters reproduce the overall body pose, but axial twist causes the head orientation to differ from the input.
Rank-two directional support constrains swing while leaving axial twist undetermined by bone directions alone.

CAST assumes fixed parent-local bone offsets and represents skeletal motion solely through joint rotations.
Under this assumption, forward kinematics changes joint positions by rotating the skeleton while preserving the offsets specified by its rest pose.
Some animations also translate bones relative to their parents, causing these local offsets to vary over time.
Such motion lies outside the rotation-only representation, so even accurate rotation predictions may be insufficient to reproduce the complete animation.
Supporting these motions would require extending CAST to predict time-varying local translations alongside joint rotations.

\section{Additional Qualitative Results}
\label{app:qualitative_supplement}
We supplement Figure~\ref{fig:qualitative_results} with matched-rig examples on Zoo and MObj, followed by Mixamo cross-skeleton results, cross-skeleton transfer between Zoo and MObj skeletons, and masked-frame completion examples.
Dataset and split sizes are detailed in Appendix~\ref{app:data}.
In the matched-rig panels, each example compares methods over five frames, with time advancing from top to bottom.

\subsection{Zoo}
\label{app:qualitative_zoo}
Figures~\ref{fig:qualitative_zoo_1} and~\ref{fig:qualitative_zoo_2} show the matched-rig Zoo examples.
\begin{figure*}[p]
\centering
\includegraphics[page=1,width=\linewidth,height=0.90\textheight,keepaspectratio]{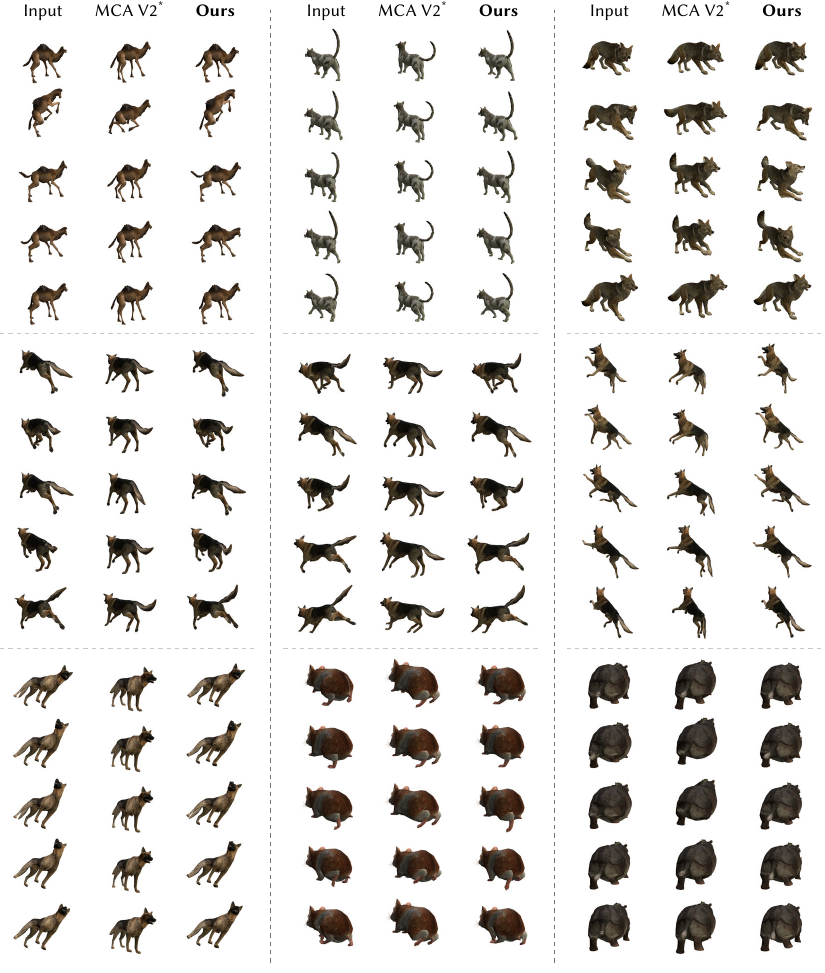}
\caption{\textbf{Matched-rig motion capture on Zoo.} Columns within each panel show Input, MCA V2$^{\ast}$, and Ours. All methods use the same rig, and successive rows show the motion advancing over time.}
\label{fig:qualitative_zoo_1}
\end{figure*}

\begin{figure*}[p]
\centering
\includegraphics[page=2,width=\linewidth,height=0.90\textheight,keepaspectratio]{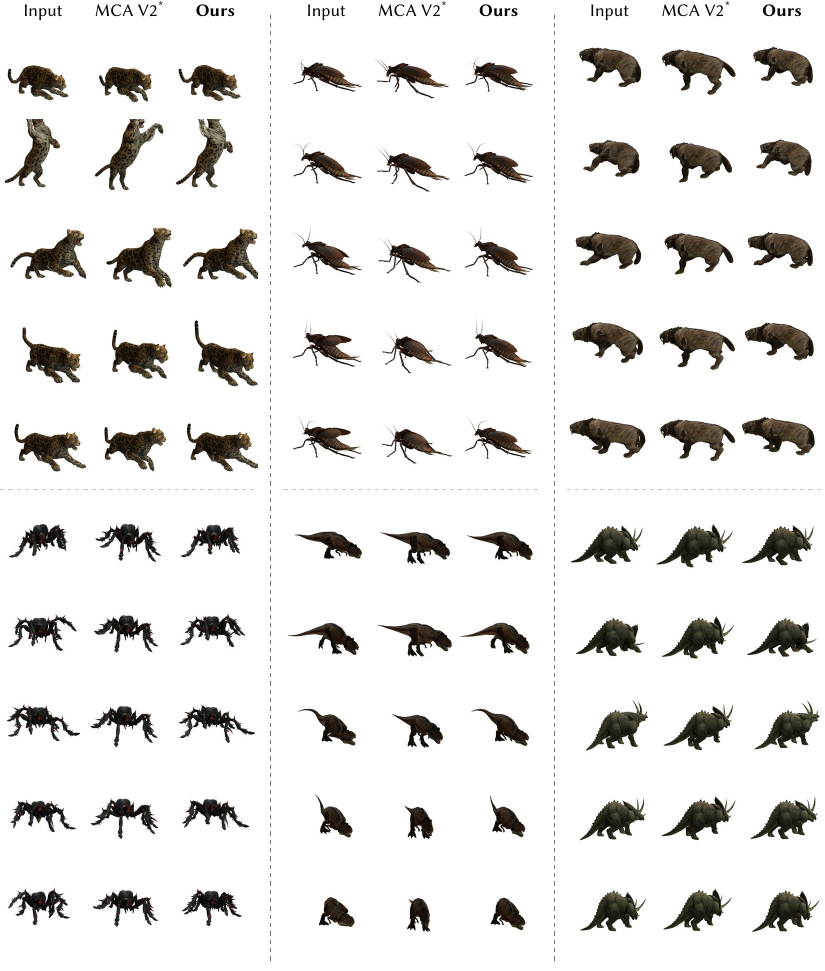}
\caption{\textbf{Matched-rig motion capture on Zoo (continued).} Columns within each panel: Input, MCA V2$^{\ast}$, and Ours. Rows show five successive sampled frames.}
\label{fig:qualitative_zoo_2}
\end{figure*}

\subsection{MObj}
\label{app:qualitative_obj}
Figures~\ref{fig:qualitative_obj_1}--\ref{fig:qualitative_obj_4} show the matched-rig MObj examples.
\begin{figure*}[p]
\centering
\includegraphics[page=1,width=\linewidth,height=0.90\textheight,keepaspectratio]{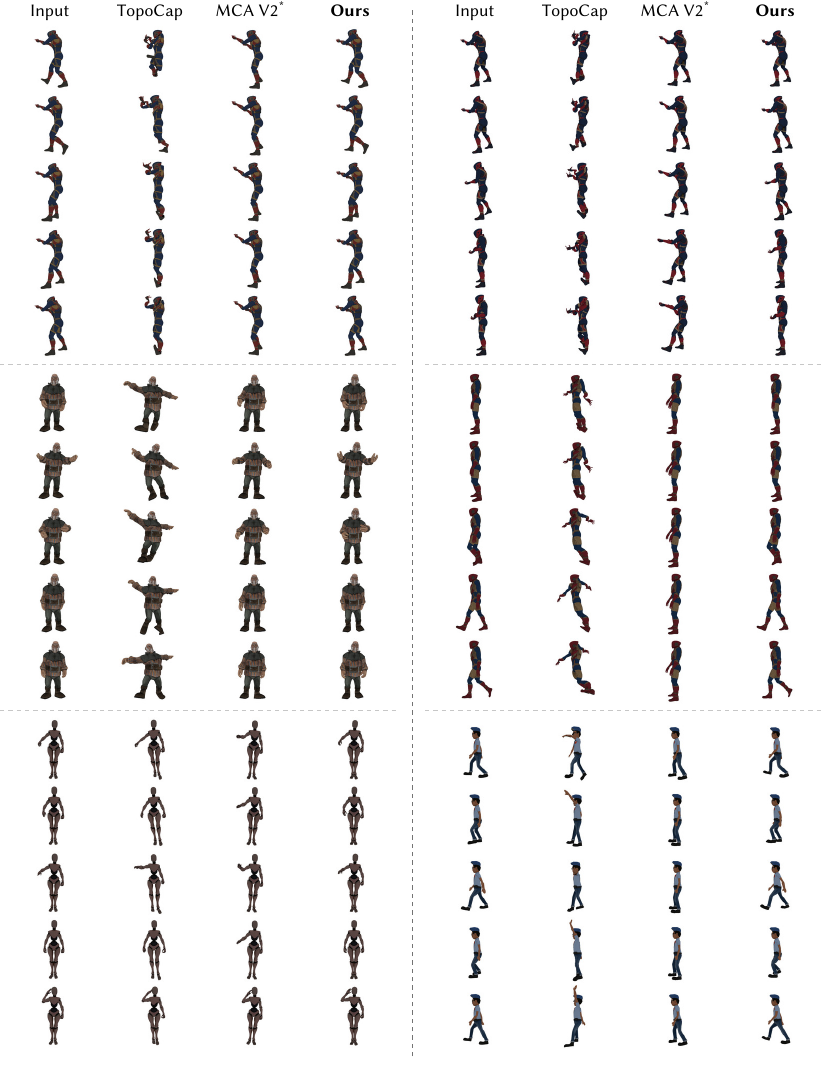}
\caption{\textbf{Matched-rig motion capture on MObj.} Columns within each panel show Input, TopoCap, MCA V2$^{\ast}$, and Ours, evaluated on the same rig. Time advances from top to bottom.}
\label{fig:qualitative_obj_1}
\end{figure*}

\newcommand{\objqualpage}[1]{%
\begin{figure*}[p]
\centering
\includegraphics[page=#1,width=\linewidth,height=0.90\textheight,keepaspectratio]{figures/qualitative_supplement_obj.pdf}
\caption{\textbf{Matched-rig motion capture on MObj (continued).} Columns within each panel: Input, TopoCap, MCA V2$^{\ast}$, and Ours. Rows show five successive sampled frames.}
\label{fig:qualitative_obj_#1}
\end{figure*}

}
\objqualpage{2}
\objqualpage{3}
\objqualpage{4}

\subsection{Mixamo cross-skeleton results}
\label{app:qualitative_mixamo}
Each action group shows the source-video input followed by CAST (Ours) on five target characters. The input is separated by a vertical rule, and rows show five successive frames from top to bottom.

\begin{figure*}[p]
\centering
\includegraphics[page=1,width=\linewidth,height=0.90\textheight,keepaspectratio]{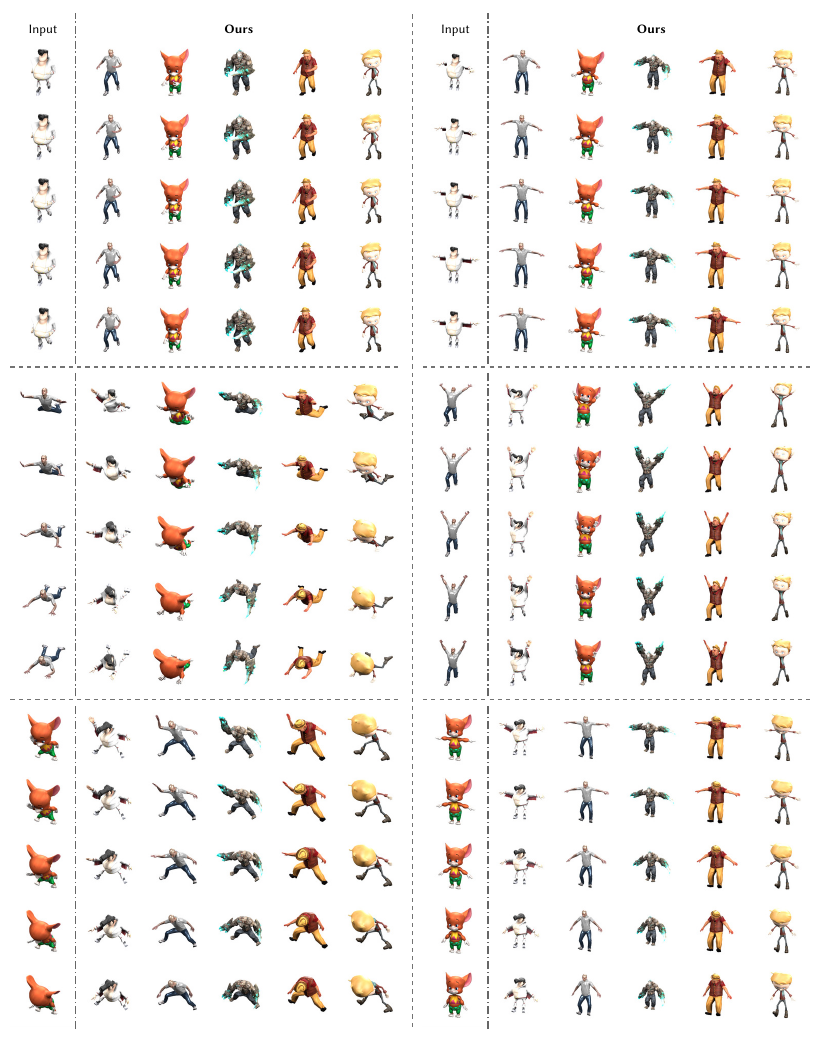}
\caption{\textbf{Cross-skeleton motion capture on Mixamo.} Each action group shows the input and CAST (Ours) on five target characters over five successive frames.}
\label{fig:qualitative_mixamo_1}
\end{figure*}

\newcommand{\mixamoqualpage}[1]{%
\begin{figure*}[p]
\centering
\includegraphics[page=#1,width=\linewidth,height=0.90\textheight,keepaspectratio]{figures/qualitative_supplement_mixamo.pdf}
\caption{\textbf{Cross-skeleton motion capture on Mixamo (continued).} Each action group shows the input and CAST (Ours) on five target characters over five successive frames.}
\label{fig:qualitative_mixamo_#1}
\end{figure*}

}
\mixamoqualpage{2}

\subsection{Cross-skeleton transfer}
\label{app:qualitative_cross_skeleton}
Each example below is one source/target skeleton pair: the left column (Input) shows a frame of the source clip and the right column (Ours) shows the motion CAST produces on the target skeleton, over five successive frames.
Each direction is illustrated with twenty source/target pairs.
Because no retargeted reference motion is available for these pairs, the transfer is presented qualitatively and no metrics are reported for them.
Quantitative cross-skeleton results are given on the Mixamo benchmark (Appendix~\ref{app:evaluation_protocol}).

\begin{figure*}[p]
\centering
\includegraphics[page=1,width=\linewidth]{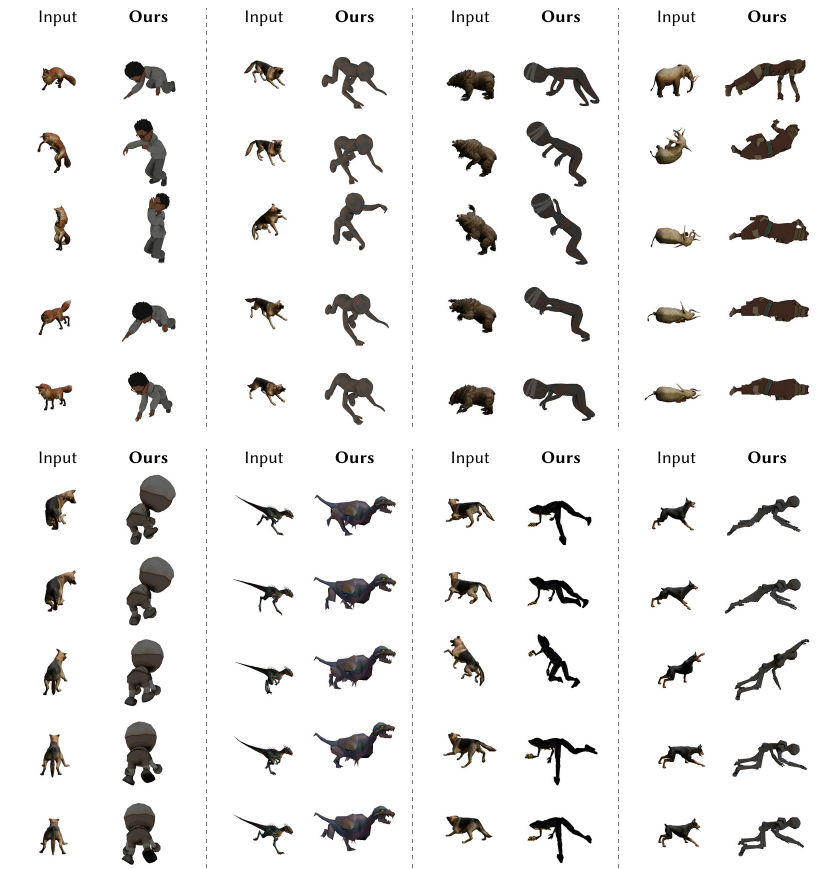}
\caption{\textbf{Cross-skeleton motion capture from Zoo to MObj.} Qualitative examples of this transfer direction, each shown as five successive frames of the source clip (Input) and our result on the target skeleton (Ours).}
\label{fig:qualitative_cross_zoo2obj_1}
\end{figure*}

\begin{figure*}[p]
\centering
\includegraphics[page=2,width=\linewidth]{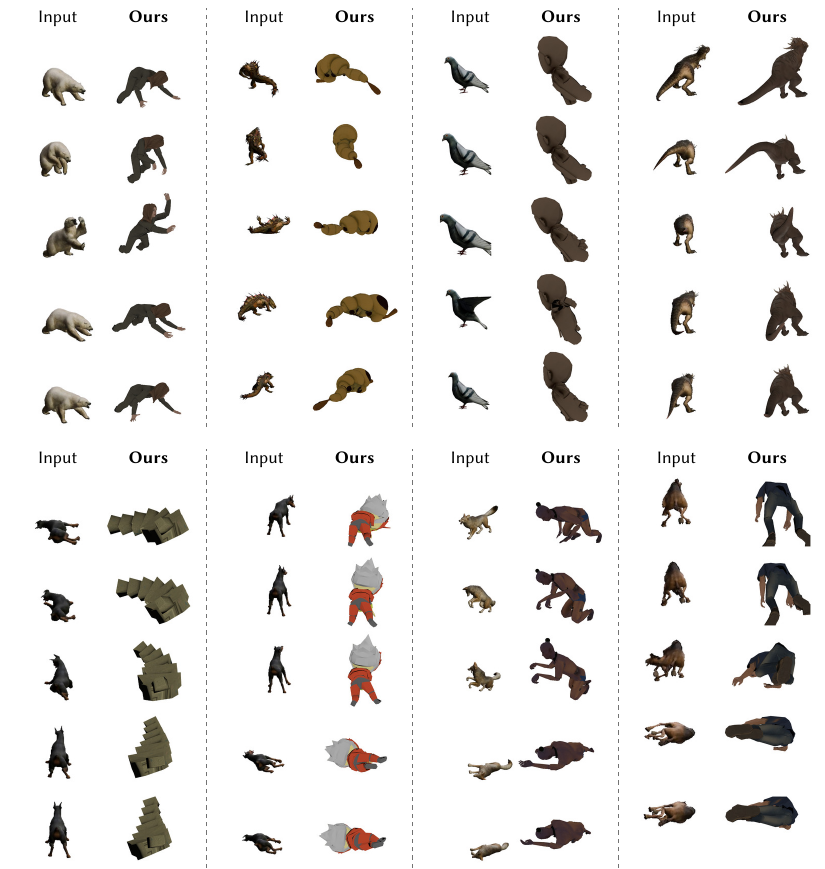}
\caption{\textbf{Cross-skeleton motion capture from Zoo to MObj (continued).} Further qualitative examples for the same transfer direction.}
\label{fig:qualitative_cross_zoo2obj_2}
\end{figure*}

\begin{figure*}[p]
\centering
\includegraphics[page=3,width=\linewidth]{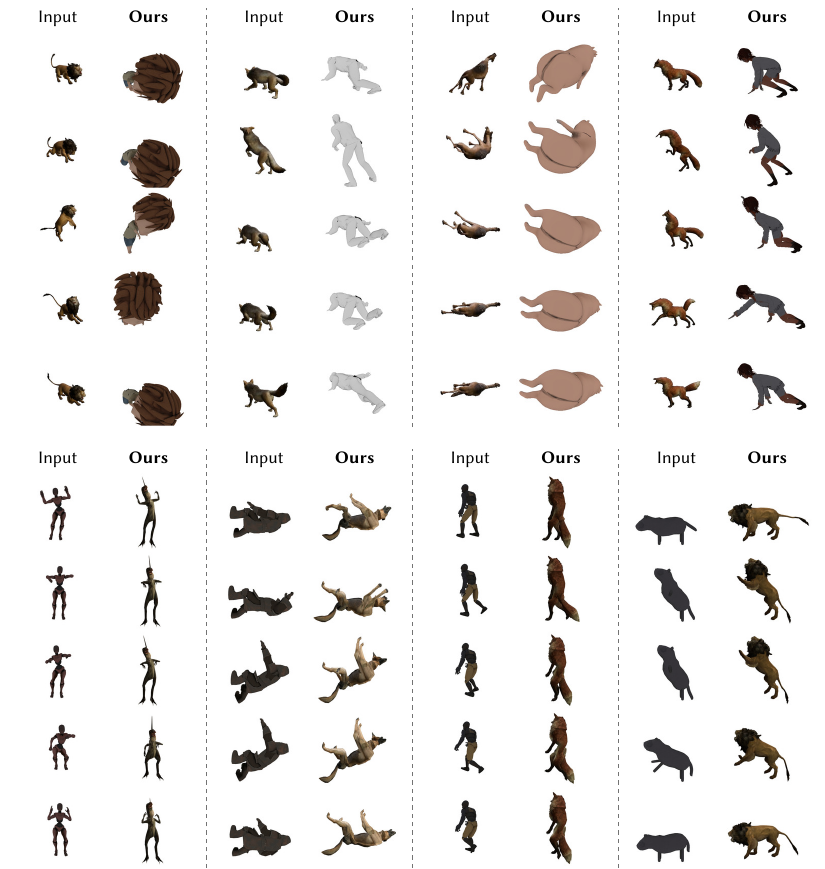}
\caption{\textbf{Cross-skeleton motion capture between Zoo and MObj.} Top: remaining Zoo-to-MObj examples. Bottom: MObj-to-Zoo examples. Each pair shows five successive frames of the source clip (Input) and our result on the target skeleton (Ours).}
\label{fig:qualitative_cross_zoo2obj_3}
\label{fig:qualitative_cross_obj2zoo_1}
\end{figure*}

\begin{figure*}[p]
\centering
\includegraphics[page=4,width=\linewidth]{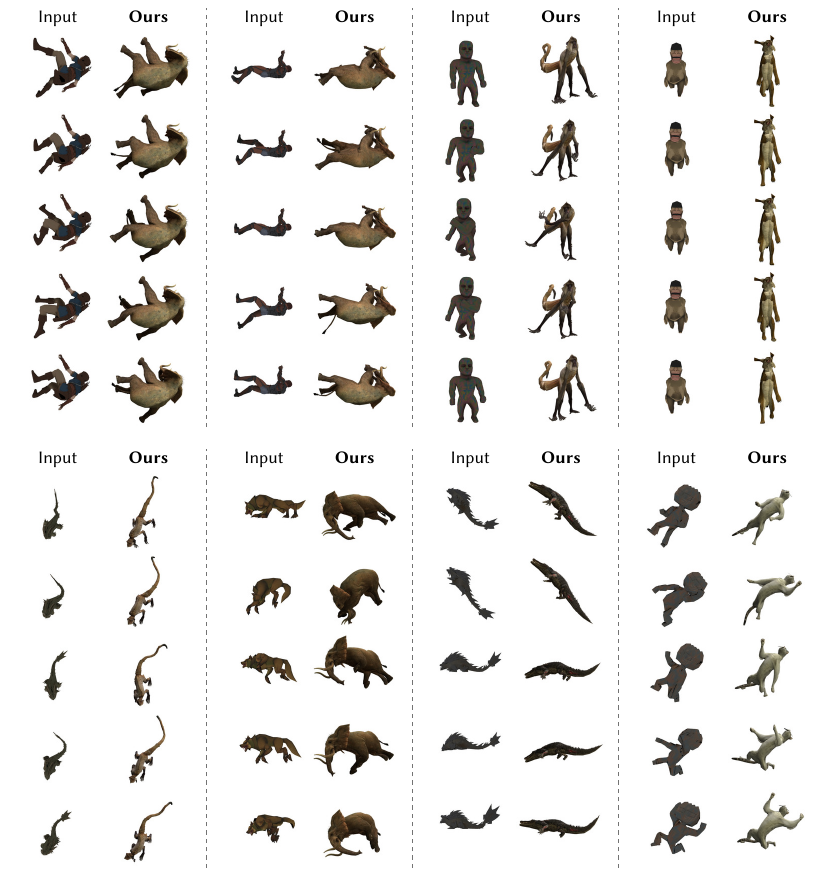}
\caption{\textbf{Cross-skeleton motion capture from MObj to Zoo (continued).} Further qualitative examples for the same transfer direction.}
\label{fig:qualitative_cross_obj2zoo_2}
\end{figure*}

\begin{figure*}[p]
\centering
\includegraphics[page=5,width=\linewidth]{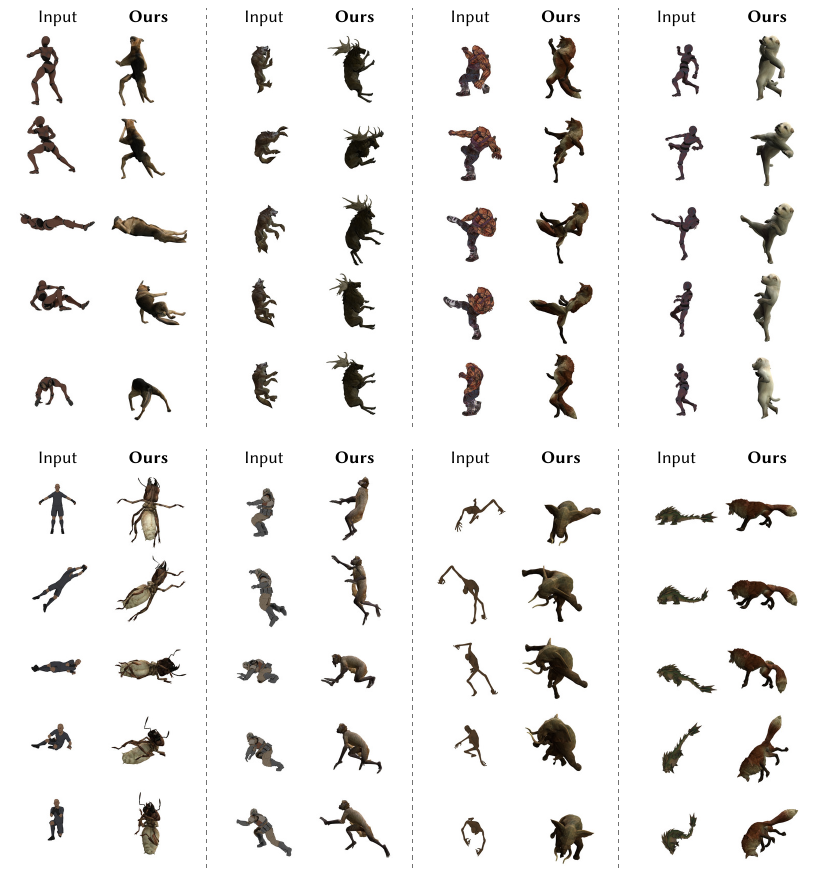}
\caption{\textbf{Cross-skeleton motion capture from MObj to Zoo (continued).} Remaining qualitative examples for the same transfer direction.}
\label{fig:qualitative_cross_obj2zoo_3}
\end{figure*}

\begin{figure*}[p]
\centering
\includegraphics[page=6,width=\linewidth]{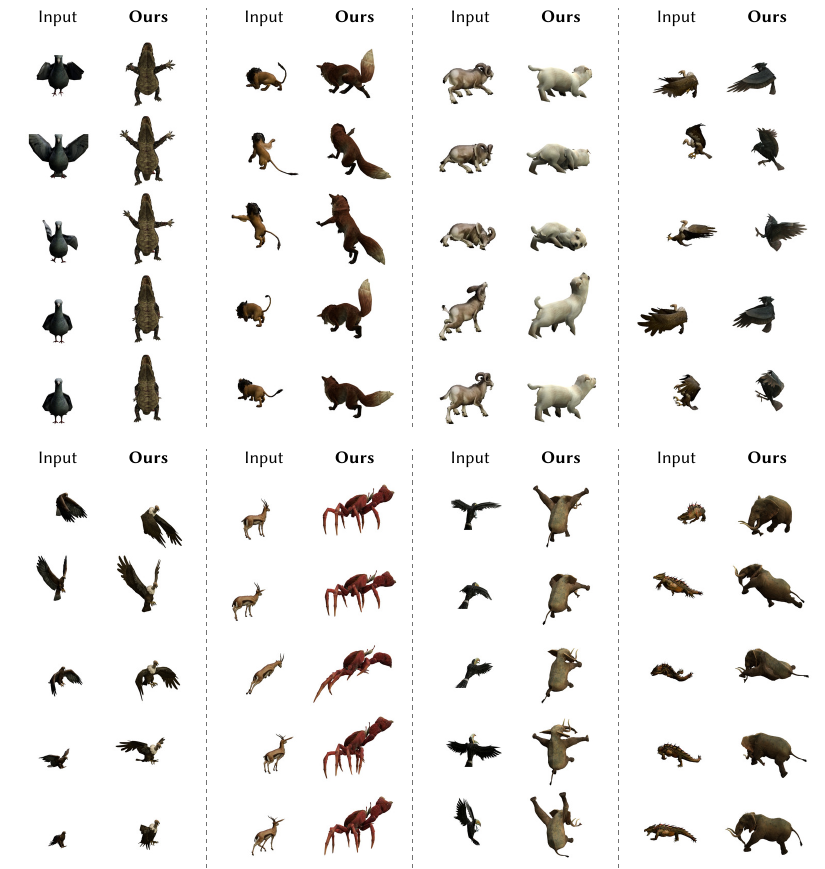}
\caption{\textbf{Cross-skeleton motion capture within Zoo.} Qualitative examples of this transfer direction, each shown as five successive frames of the source clip (Input) and our result on the target skeleton (Ours).}
\label{fig:qualitative_cross_zoo2zoo_1}
\end{figure*}

\begin{figure*}[p]
\centering
\includegraphics[page=7,width=\linewidth]{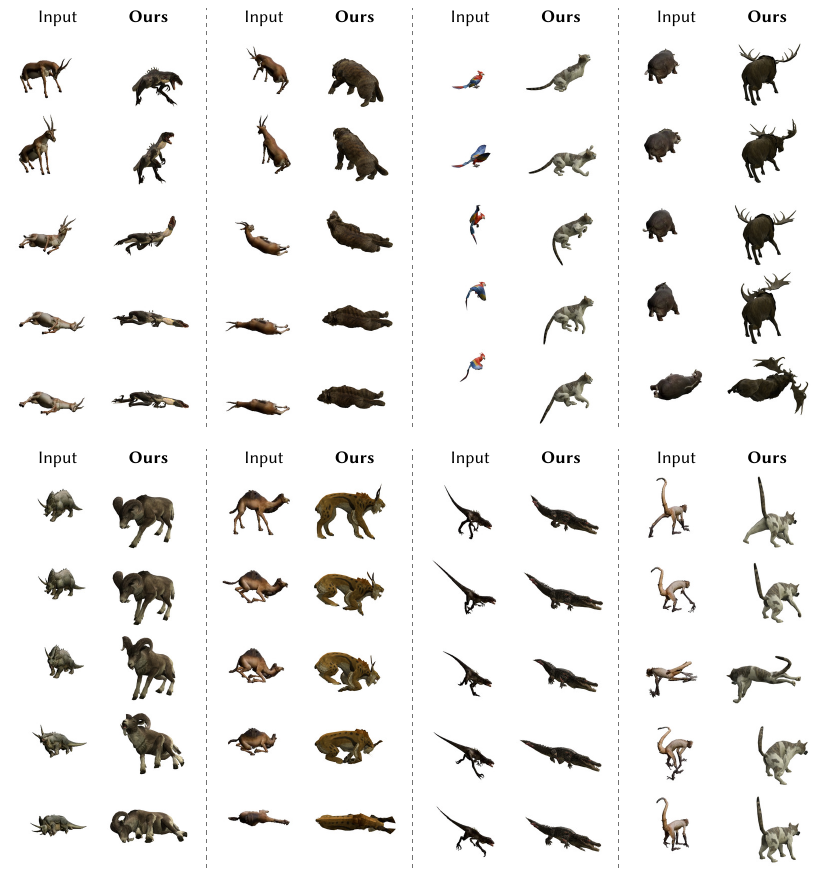}
\caption{\textbf{Cross-skeleton motion capture within Zoo (continued).} Further qualitative examples for the same transfer direction.}
\label{fig:qualitative_cross_zoo2zoo_2}
\end{figure*}

\begin{figure*}[p]
\centering
\includegraphics[page=8,width=\linewidth]{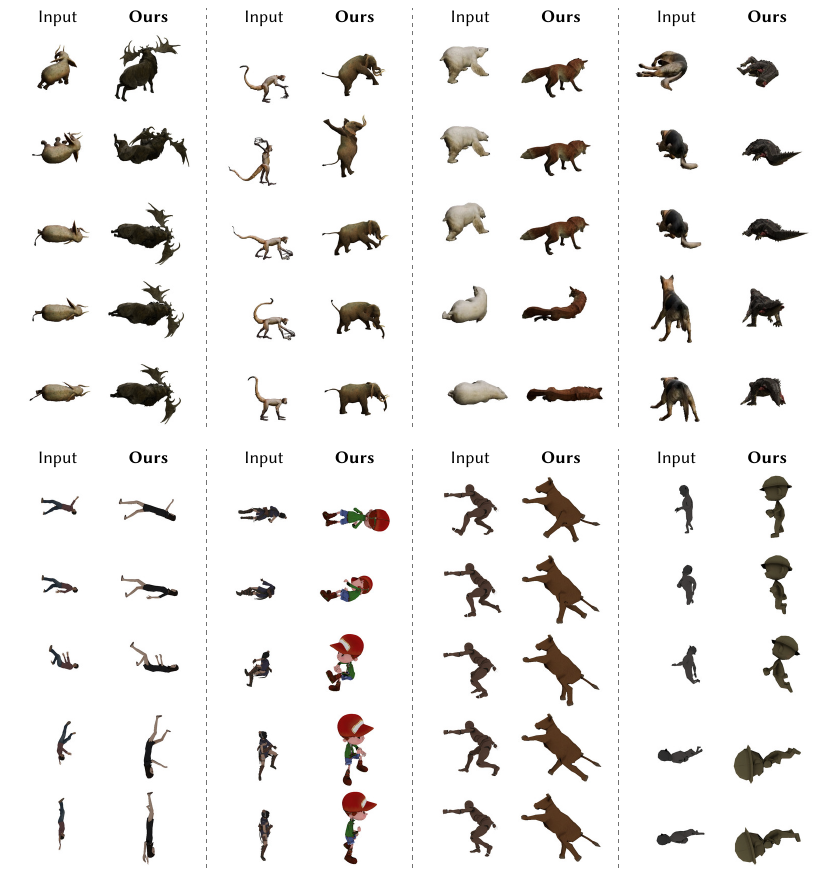}
\caption{\textbf{Cross-skeleton motion capture within Zoo and MObj.} Top: remaining Zoo-to-Zoo examples. Bottom: MObj-to-MObj examples. Each pair shows five successive frames of the source clip (Input) and our result on the target skeleton (Ours).}
\label{fig:qualitative_cross_zoo2zoo_3}
\label{fig:qualitative_cross_obj2obj_1}
\end{figure*}

\begin{figure*}[p]
\centering
\includegraphics[page=9,width=\linewidth]{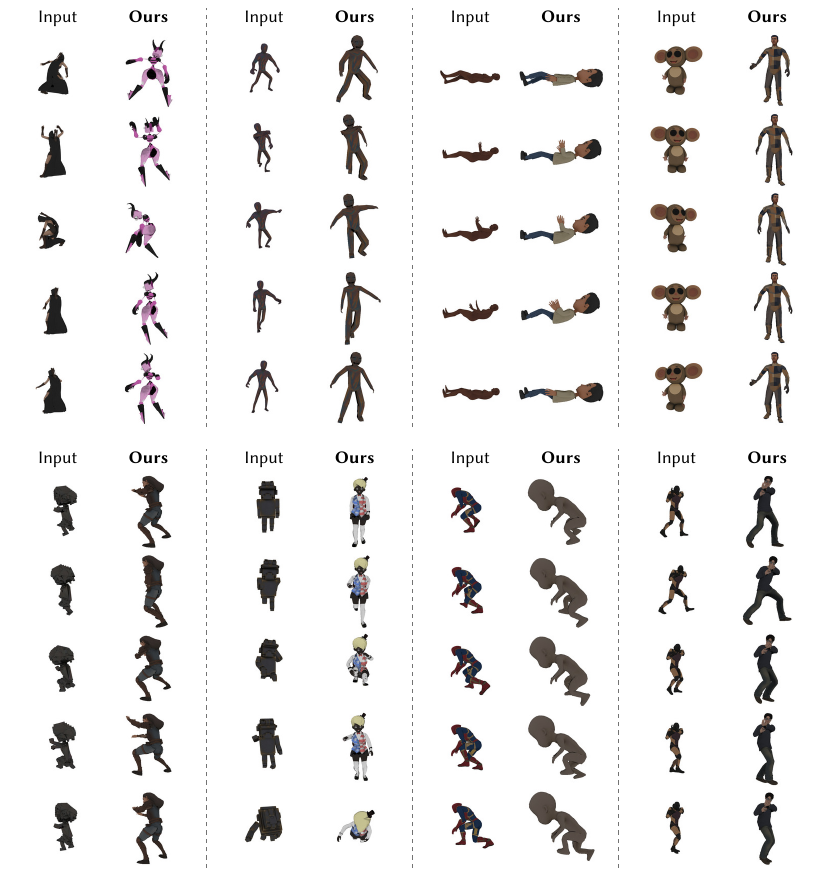}
\caption{\textbf{Cross-skeleton motion capture within MObj (continued).} Further qualitative examples for the same transfer direction.}
\label{fig:qualitative_cross_obj2obj_2}
\end{figure*}

\begin{figure*}[p]
\centering
\includegraphics[page=10,width=\linewidth]{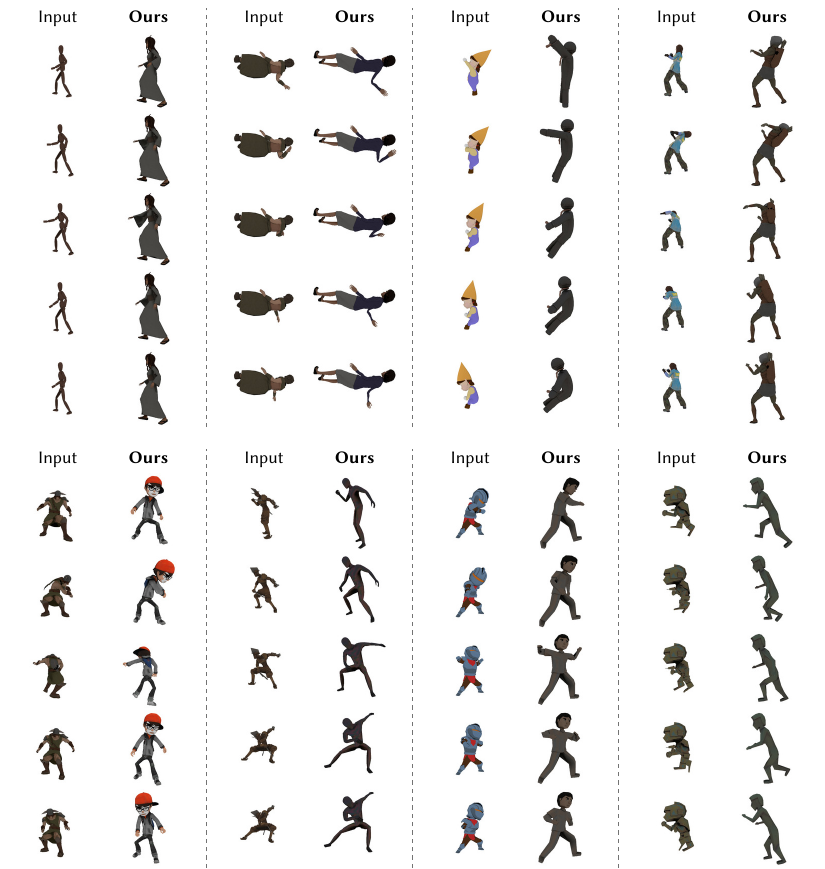}
\caption{\textbf{Cross-skeleton motion capture within MObj (continued).} Remaining qualitative examples for the same transfer direction.}
\label{fig:qualitative_cross_obj2obj_3}
\end{figure*}

\subsection{Masked-frame completion examples}
\label{app:completion_examples}
Figures~\ref{fig:completion_examples}--\ref{fig:completion_examples_3} show the masked-frame completion examples.
\begin{figure*}[p]
\centering
\includegraphics[width=\linewidth,height=0.88\textheight,keepaspectratio]{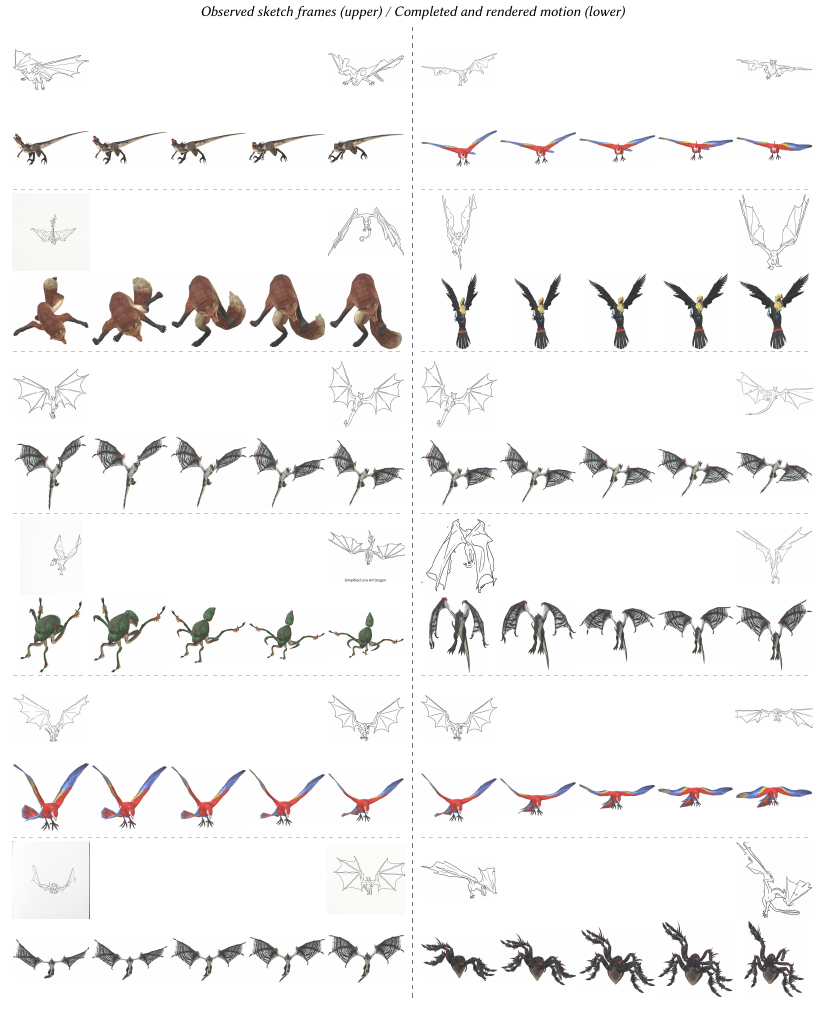}
\caption{\textbf{Masked-frame completion examples.} The first twelve of thirty-five examples drawn from 22 animation sequences. Each example shows exactly two observed sketch frames in the upper row, aligned with the first and last frames of the completed five-frame interval shown below. At inference, one image is observed every four frames and the remaining frames are completed by CAST.}
\label{fig:completion_examples}
\end{figure*}

\newcommand{\completionqualpage}[2]{%
\begin{figure*}[p]
\centering
\includegraphics[width=\linewidth,height=0.90\textheight,keepaspectratio]{figures/cast_motion_completion_#1.pdf}
\caption{\textbf{Masked-frame completion examples (continued).} #2}
\label{fig:completion_examples_#1}
\end{figure*}

}
\completionqualpage{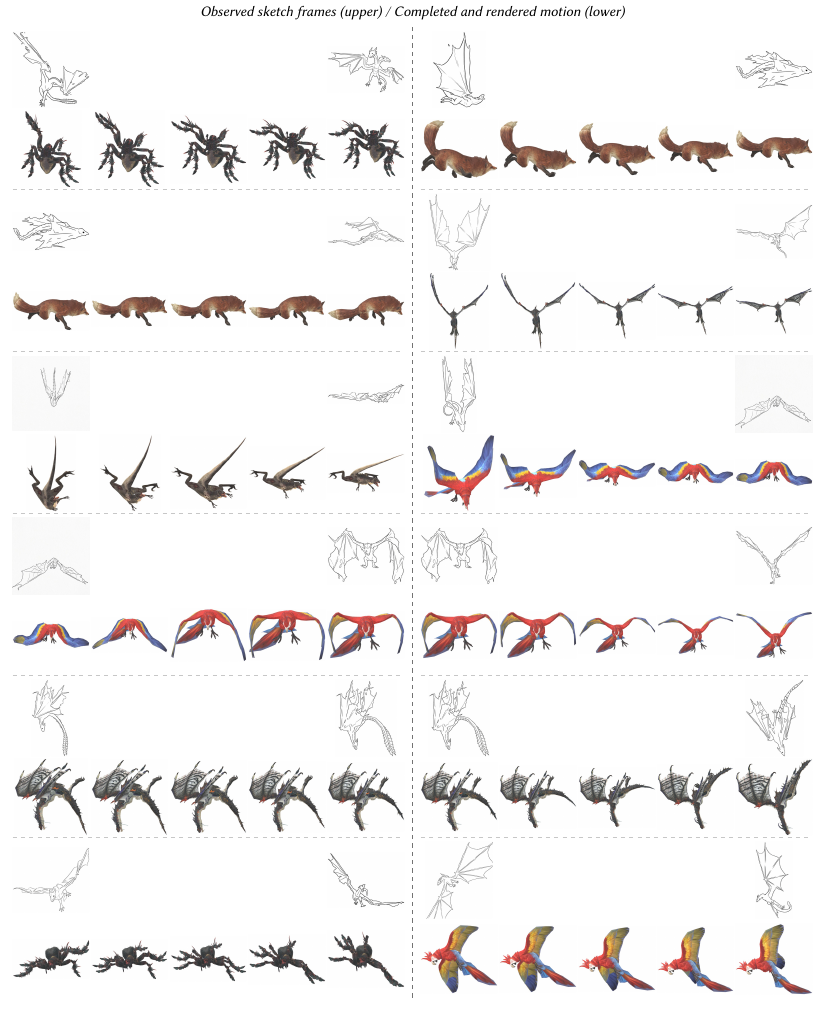}{The next twelve of the thirty-five examples, using the same two-endpoint observation setting.}
\completionqualpage{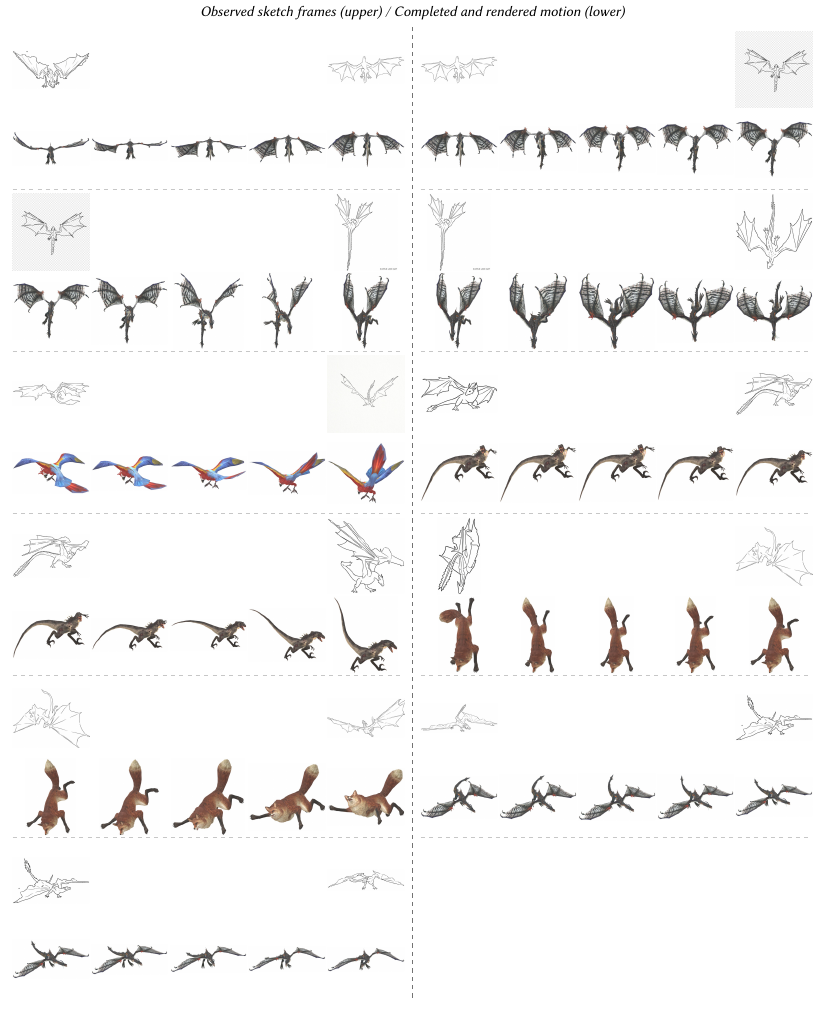}{The remaining eleven of the thirty-five examples, using the same two-endpoint observation setting.}

\end{document}